\documentclass{article}

\usepackage{nac_preprint}
\usepackage[utf8]{inputenc} 
\usepackage[T1]{fontenc}    
\usepackage{hyperref}       
\usepackage{url}            
\usepackage{booktabs}       
\usepackage{microtype}      
\usepackage{booktabs} 
\usepackage{amsmath}
\usepackage{amssymb}
\usepackage{pifont}
\usepackage[toc,page]{appendix}
\usepackage{enumitem}
\usepackage{svg}
\usepackage{graphics}
\usepackage{graphicx}
\usepackage{algorithm}
\usepackage{pdfpages}
\usepackage{algorithmicx}
\usepackage[noend]{algpseudocode}
\usepackage{color}
\usepackage{dsfont}
\usepackage{tabularx}
\usepackage{subcaption}
\DeclareMathOperator*{\argmin}{arg\,min}
\DeclareMathOperator*{\argmax}{arg\,max}
\usepackage{algorithm}
\usepackage{algorithmicx}
\usepackage[noend]{algpseudocode}
\usepackage{stmaryrd}
\newtheorem{theorem}{Theorem}[section]

\newtheorem{proposition}[theorem]{Proposition}

\newtheorem{lemma}[theorem]{Lemma}

\newenvironment{proof}{\paragraph{Proof:}}{\hfill$\square$}

\newcommand{\xmark}{\ding{55}}

\title{Neuro-mimetic Task-free Unsupervised Online Learning with Continual Self-Organizing Maps}

\author{
 Hitesh Vaidya \\
  Department of Computer Science and Engineering\\
  University of South Florida\\
  Tampa, FL 33620 \\
  \texttt{hvaidya@usf.edu} \\
  \And
  Travis Desell \\
  Department of Software Engineering\\
  Rochester Institute of Technology\\
  Rochester, NY 14623 \\
  \texttt{tjdvse@rit.edu} \\
   \And
 Ankur Mali \\
  Department of Computer Science and Engineering\\
  University of South Florida\\
  Tampa, FL 33620 \\
  \texttt{ankurarjunmali@usf.edu} \\
  \And
 Alexander Ororbia \\
  Department of Computer Science \\
  Rochester Institute of Technology\\
  Rochester, NY 14623 \\
  \texttt{agovcs@rit.edu} \\
}

\begin{document}
\setlength{\abovedisplayskip}{0.065cm}
\setlength{\belowdisplayskip}{0pt}

\maketitle

\begin{abstract}
An intelligent system capable of continual learning is one that can process and extract knowledge from potentially infinitely long streams of pattern vectors. The major challenge that makes crafting such a system difficult is known as \emph{catastrophic forgetting} -- an agent, such as one based on artificial neural networks (ANNs), struggles to retain previously acquired knowledge when learning from new samples. Furthermore, ensuring that knowledge is preserved for previous tasks becomes more challenging when input is not supplemented with task boundary information. Although forgetting in the context of ANNs has been studied extensively, there still exists far less work investigating it in terms of unsupervised architectures such as the venerable self-organizing map (SOM), a neural model often used in clustering and dimensionality reduction. While the internal mechanisms of SOMs could, in principle, yield sparse representations that improve memory retention, we observe that, when a fixed-size SOM processes continuous data streams, it experiences concept drift. In light of this, we propose a generalization of the SOM, the continual SOM (CSOM), which is capable of online unsupervised learning under a low memory budget. Our results, on benchmarks including MNIST, Kuzushiji-MNIST, and Fashion-MNIST, show almost a two times increase in accuracy, and CIFAR-10 demonstrates a state-of-the-art result when tested on (online) unsupervised class incremental learning setting.
\end{abstract}

\section{Introduction}
\label{sec:intro}

A major hurdle in designing autonomous, continuously learning agents is ensuring that such systems effectively preserve previously acquired knowledge when faced with new information. The difficulty that agents face in retaining information acquired over time is known as catastrophic forgetting (or interference) \cite{mccloskey_catastrophic_1989,ratcliff_connectionist_1990,mccloskey_catastrophic_1989} and is a significant challenge in the problem setting known as lifelong or continual learning \cite{thrun1995lifelong,chen2016lifelong,ClassIncrementalSurvey,ororbia2020continual, ororbia2022lifelong}. In continual learning, an agent must embed knowledge learned from data in a sequential manner without compromising prior knowledge. Much as humans do, this agent should process samples from these sources online, consolidating and transferring the knowledge acquired over time without forgetting previously learned tasks.

Notably, catastrophic interference has been investigated with respect to deep neural networks (DNNs), which are often trained to solve supervised prediction tasks, especially in efforts such as \cite{kirkpatrick2017overcoming,lopez2017gradient,generativeReplay,ororbia2020continual, ororbia2022lifelong}. However, a relative dearth of work exists concerning less mainstream neural systems, particularly those that conduct unsupervised learning such as self-organizing maps (SOMs) \cite{kohonen1982self}. In this work, we seek to rectify this gap by, first, studying the degree to which forgetting occurs in a classical model such as the SOM and, second, developing a generalization of this system, which we call the Continual SOM or Continual Kohonen map (\emph{CSOM}), which is robust to the interference encountered in the context of online continual learning.

In essence, SOMs are a type of brain-inspired (or neuro-mimetic) unsupervised neural system where its neuronal units compete for the right to activate in the presence of particular input patterns, and the synaptic parameters associated with the winning unit(s) are adjusted via a form of Hebbian learning  \cite{hebb1949organization,martinetz1993competitive,ororbia2023brain}. Notably, the SOM's units are often arranged in either a spatial format, e.g., in a Cartesian plane/grid, or in a topological fashion, e.g., in a neighborhood/field based on the Euclidean distance between the activation values of units themselves. A useful property of the SOM is that its topologically-arranged neuronal units effectively learn to construct a low-dimensional ``semantic'' map of the more complicated sensory input space, where similar data patterns are grouped more closely together (around particular ``prototypes'') and farther apart from more dissimilar ones. This makes SOMs quite useful for clustering and dimensionality reduction tasks \cite{vesanto2000clustering,baccao2005self,bigdeli2022application}.

At first glance, it would appear that a competitive learning model such as the SOM might offer a natural immunity to forgetting since any particular neuronal unit tends to activate more often than others (by virtue of the distance function) in response to similar sensory input patterns throughout the learning process. This specialization would, in principle, result in non-overlapping, sparse representations, which have been argued to be a key way for reducing neural cross-talk  \cite{srivastava2013compete,ororbia2021continual}, a source of forgetting  \cite{mccloskey_catastrophic_1989,ratcliff_connectionist_1990,french1999catastrophic}. Furthermore, some approaches, such as \cite{gepperth2015bioincrement}, have advocated for using SOMs in the continual learning setting. Nevertheless, as uncovered in the experiments of this paper, the SOM, in its purest form, appears to be prone to forgetting, thus motivating our particular model generalization and computational framework \cite{frontiersForgetting}. 

In service of the problem of online continual learning, the core contributions of this paper can be summarized as:
\begin{itemize}
    \item We adapt the classical SOM model, which was initially formulated for fitting to single static datasets, to the online continual learning setting; specifically, we study its memory retention in tandem with its ability to adapt to pattern streams. Our experiments show that the SOM experiences substantial interference across all benchmarks examined.
    \item We develop a generalization of the SOM, termed the CSOM, which experiences significantly less forgetting by introducing mechanisms such as specialized decay functionality and running variance to select the best matching unit (BMU) for an input at a given time step.
    \item We present experimental results for four class-incremental datasets, i.e., MNIST, Fashion MNIST, Kuzushiji-MNIST, and CIFAR-10, and empirically demonstrate the robustness of the proposed CSOM to forgetting, yielding a promising, unsupervised neuromimetic system for the lifelong learning setting.
\end{itemize}

\section{Related Work}
\label{sec:lit_review}

In this section, we review prior work related to the central topics that drive this paper: general principles of competitive learning, self-organizing maps, and continual unsupervised learning. 


\noindent
\textbf{Competitive Learning.} In competitive learning \cite{Hartono2012}, neurons compete with each other to best match with encountered input sample patterns. This form of neural computation is often built on top of Hebbian learning \cite{hebb1949organization,Choe2013,ororbia2023brain} (where an adjustment made to a synapse depends on only local information that is available to it both spatially and temporally)  and is typically used to perform clustering on input data. Vector quantization \cite{gray1984vector} and self-organizing maps (Kohonen maps) \cite{kohonen1982self} are prominent examples within this family of models. Unlike modern-day DNNs, where all internal neurons participate in every step of inference and learning, in competitive learning, only the neurons that satisfy certain criteria ``win'' the right to compute \cite{srivastava2013compete} and update their connection weight strengths. This form of learning can be useful in identifying and extracting useful features within a dataset. According to \cite{rumelhart1986parallel,ororbia2023brain}, three fundamental elements define competitive learning in general: \textbf{1}) the model starts with a set of units that are highly similar, except for some random noise, which makes each of them respond slightly differently to a set of inputs; 
\textbf{2}) there is a limit to the ``strength'' of each unit (which motivates the notion of neighborhood functions, as discussed later); and 
\textbf{3}) the units are allowed to compete in some way for the right to respond to a particular subset of inputs.

\noindent
\textbf{Self-Organizing Maps.} The self-organizing map (SOM) \cite{kohonen1982self}, and its many variants \cite{khacef2020improving,rougier2021randomized,kopczynski2021non,gliozzi2018self}, is an unsupervised clustering neural model that adjusts its connection strengths via a Hebbian update rule \cite{hebb1949organization}. During training, spatially arranged clusters gradually form around the best-matching neurons within the SOM. This allows the SOM to be useful as a data exploration tool and even as an effective minimally supervised learner \cite{lyu2024minimally}, wherein its internal units represent a summary of the latent patterns found within a dataset. In contrast to clustering algorithms, such as K-means \cite{forgy1965cluster,lloyd1982least}, SOMs perform a soft clustering over inputs, which means that the connection weight update for the best matching neuron is the strongest while the updates made to the others decay/fade gradually as one moves farther away to other neurons within the best matching neuron's neighborhood. 

SOMs, in the view of this study, make potentially invaluable memory systems of sensory input patterns given that they iteratively compress their information into compact parameter vectors (or neural ``templates'') as opposed to the raw data buffers \cite{DERverse, MemoryEvolution} generally used for storing clusters or exemplars of specific task datasets, much like those often used in many continual learning methods \cite{8793982,lopez2017gradient}. Furthermore, as has been shown in prior work \cite{bashivan2019continual,pinitas2021dendritic}, the SOM's learning process is better equipped to facilitate the capture of the subtle differences (or variance) across a task's constituent data patterns.

\noindent
\textbf{Continual Unsupervised Learning.} 
One of the applications of continual learning is to design agents capable of performing intelligent operations on low-resource or edge-computing devices, such as those generally found in self-driving cars \cite{liu2019edge} or robotic control systems. Like many other real-world sources, self-driving cars generate enormous quantities of data that are often unlabeled or unannotated. Therefore, the value of unsupervised learning is relatively high for these kinds of applications. A consequence is that effort will be required to obtain optimized solutions to tackle the problem of catastrophic forgetting in unsupervised learning. 

Crafting probabilistic generative models \cite{generativeReplay, SelfSupervisedTeacherStudent} that are capable of synthesizing data samples, which can be used to refresh the memory of task-specialized neural models (e.g., a classifier/regressor), is one prevalent approach. \cite{ayub2021eec} employed (neural) encoder models to store the centroids of task data, which were later used to generate data from previous tasks in order to induce replay. However, this approach is not computationally feasible for large quantities of data -- storing multiple encoder models per task results in an increase in memory complexity as well as incurs additional computational time needed to train each additional new encoder network for every newly encountered task. This greatly hinders the scalability of such an approach. In general, rehearsal-based approaches \cite{8100070,Isele2018SelectiveER}, of which some schemes could technically be considered to fall under the umbrella of unsupervised continual learning, often store data samples in some form of explicit memory. However, low-resource devices often do not have a large memory capacity to store (enough) samples that adequately capture the variance of every task's distribution; therefore, this reduces such schemes' effectiveness given that explicitly storing more data to facilitate effective (memory-based) refreshing/retraining of more complex neural predictor model is not feasible \cite{VariationalDistillationCL}. 

In this work, we will show that our proposed CSOM can readily and usefully capture the variance inherent to each task's dataset (within a sequence) and yet not need constant, expensive retraining or refreshing itself. Finally, although out of scope for this paper, we remark that the SOM models we study could also be made to expand much as in (growing, \cite{AdaptiveProgressiveCL, SelfSupervisedTeacherStudent, SupportNetwork}) dynamical neural gas models \cite{10.5555/2998687.2998765,VENTOCILLA2021100254}

\section{Methodology} 
\label{sec:method}

\noindent
\textbf{Notation.} We start by defining the notation that will be used throughout this paper. $\odot$ indicates a Hadamard product, $\cdot$ indicates a matrix/vector multiplication (or dot product if the two objects it is applied to are vectors of the same shape). $||\mathbf{v}||_p$ is used to represent the $p$-norm (distance function w.r.t. the difference between two vectors), i.e., $p = 2$ selects the $2$-norm or Euclidean distance. 
$\mathbf{W}[:,i]$ is the slice operator, meant to extract the $i$th column vector of matrix $\mathbf{W}$; $\mathbf{W}[i,:]$ is meant to extract the $i$th row. $\cos(\theta)$ indicates cosine similarity.

\noindent
\textbf{Problem Definition.} Consider a sequence of $T$ tasks, which we formally denote by $\mathcal{S} = \bigcup_{k=1}^T \mathcal{T}_k$. Each task $\mathcal{T}_k$ has a training dataset (each containing $C$ classes), $\mathcal{D}_{train}^{(k)} = \bigcup_{i=1}^{N_k} \{(\mathbf{x}_i^{(k)}, \mathbf{y}_i^{(k)})\}$, where $\mathbf{x}_i^{(k)} \in \mathcal{R}^{D \times 1}$ is a data pattern and  $\mathbf{y}_i^{(k)} \in \{0,1\}^{C \times 1}$ is its label vector. Furthermore, we let $N_k$ be the number of patterns in task $T_k$ and $\mathcal{D}_{test}^{(k)}$ to represent the test dataset for task $\mathcal{T}_k$.  We remark that, although the datasets investigated in this work come with labels, our models will never use them since they are unsupervised. However, we will use the labels for external analysis, as we will see in our defined metric(s). Finally, all models process each data point in  $\mathcal{D}^{(k)}_{train}$ \textbf{only once (online adaptation)} and time steps are tracked in simulation within the variable $t$, i.e., each time that a data point $\mathbf{x}^{(k)}_i$ is sampled from $\mathcal{D}^{(k)}_{train}$ of task $\mathcal{T}_k$, $t$ is incremented as $t \leftarrow t + 1$. 

When a continual learning agent is finished training on task $\mathcal{T}_k$ (using $\mathcal{D}_{train}^{(k)}$), the data  $\mathcal{D}_{train}^{(k)}$ will be lost as soon as the model proceeds to task $\mathcal{T}_{k+1}$ (up to $\mathcal{T}_T$). Furthermore, note that $\mathcal{D}_{test}^{(k)}$ is only used for external model evaluation. The general objective/goal is to maximize the agent's generalization performance on task $\mathcal{T}_k$ while minimizing how much its performance degrades on prior tasks $\mathcal{T}_1$ to $\mathcal{T}_{k-1}$. 

\begin{algorithm}[t]
    \caption{The inference and learning processes for the (classical) SOM, formulated for the online stream setting.}
    \label{alg:kohonen_map}
    \textbf{Input}: sample $\mathbf{x}^{(k)}_i(t)$, synaptic weights $\mathbf{M}$, topology $\mathcal{G}$, simulation time step $t$ \\
    \textbf{Parameter}: $\lambda, \sigma, \tau_\lambda, \tau_\sigma, \lambda_{t=0}\gets\lambda, 
 \sigma_0\gets\sigma$
    \begin{algorithmic} 
        \Function{Update}{$\mathbf{x}^{(k)}_i(t)$, $\mathbf{M}$, $\mathcal{G}$, $\sigma_t$, $\lambda_t$}
            \Comment{Compute L2 distances and perform weighted updates}
            \State $\mathbf{\Delta_2} = \mathbf{x}^{(k)}_i(t) - \mathbf{M}$,  \label{alg:l2_distance}
            $\mathbf{d} = ||\mathbf{\Delta_2}[:,j]||_2, j = 0, 1,...H$
            \Comment{Compute the BMU, $u$ and its neighbors, $\mathcal{N}_u$}
            \State $u = \arg\min_{j \in H} \mathbf{\mathbf{d}}$,\; $\mathcal{N}_u = \Call{GetNeigh}{u, \mathbf{M}, \mathcal{G}}$  
        \For{$v_j \in (u \cup \mathcal{N}_u)$}  \Comment{Update synapses of the BMU and its neighbors}
            \State $\mathbf{M}[:,v_j] \leftarrow \mathbf{M}[:,v_j]  + \lambda_t \phi(u, v_j, \mathcal{G}, \sigma_t) (\mathbf{\Delta}[:,v_j])$
        \EndFor
        \EndFunction    

        \Function{Train}{$\mathbf{x}^{(k)}_i(t), \mathbf{M}, \mathcal{G}, t$} \Comment Online training routine \Call{Train}{.}, at time step $t$
        \State $\lambda_t = \lambda_0 \exp{(-t / \tau_\lambda)}$
        \Comment or use Equation \ref{eq:lambda_t} 
        \State $\sigma_t = \sigma_0 \exp{(-t / \tau_\sigma)}$
        \Comment or use Equation \ref{eq:sigma_t} 
        \State $\Call{Update}{\mathbf{x}^{(k)}_i(t), \mathbf{M}, \mathcal{G}, \sigma_t, \lambda_t}$
        \Comment inference step
        \State $t \leftarrow t + 1$ \Comment Advance simulation time forward
    \EndFunction
    \end{algorithmic}
\end{algorithm}

\subsection{The Kohonen Map}
\label{sec:som}

In Algorithm \ref{alg:kohonen_map}, we present our formulation of the classical SOM (Kohonen Map) for the online processing of data points from a (task) stream. 
Each data point $\mathbf{x}^{(k)}_i$ is sampled from the task $\mathcal{T}_k$ (with dataset $\mathcal{D}^{(k)}_{train}$). The SOM, with $H$ neuronal units, adapts its synaptic matrix $\mathbf{M} \in \mathcal{R}^{D \times H}$ given input $\mathbf{x}^{(k)}_i$. In this study, the topology $\mathcal{G}$ is a square grid of dimensions $K \times L$ (i.e., $H = K * L, K=L$). $\phi(u,v_j,\mathcal{G},\sigma_t)$ indicates a neighborhood weighting function used to scale the update to unit $v_j$ with synaptic vector $\mathbf{M}[:,v_j]$. Note that the weighting function depends on the euclidean distance between the BMU ($u$) and the neighboring units ($v_j$) in the cartesian plane.

Specifically, we see that inference is conducted in the $\Call{Update}{ }$ routine of Algorithm \ref{alg:kohonen_map}. The SOM first computes the (Euclidean) distance between the current sample $\mathbf{x}^{(t)}_i$ and all internal units in $\mathbf{M}$ and then stores these distances in the distance (row) vector $\mathbf{d} \in \mathcal{R}^{1 \times H}$. The BMU is calculated by taking the argmin over $\mathbf{d}$, returning the index of the BMU as integer $u$\footnote{Note that $u$ further maps to a fixed two-dimensional coordinate $(k_u, l_u)$, where $0 < k_u \leq K$ and $0 < l_u \leq L$, since $\mathcal{G}$ is rectangular in our formulation of SOMs.} 
Finally, the indices/coordinates of the neighboring units of BMU -- any such neighbor is indexed by $v_j$ -- are returned by the $\Call{GetNeigh}{ }$ sub-routine and stored in the array $\mathcal{N}_u$. 

To update the relevant synaptic weight vectors of the SOM's matrix $\mathbf{M}$, a loop runs through unit indices $(u \bigcup \mathcal{N}_u)$ and a weighted Hebbian update is applied. Note that the Hebbian rule in Algorithm \ref{alg:kohonen_map} reuses the difference vectors stored in $\mathbf{\Delta}$ and applies both a  dynamic learning rate $\lambda_t$ and a coefficient produced by $\phi(u, v_j, \mathcal{G}, \sigma_t)$. The neighborhood function $\phi(u, v_j, \mathcal{G}, \sigma_t)$ was set to be a Gaussian function, centered around the BMU indexed by $u$, with radius $\sigma_t$. Note that the \textbf{classical SOM maintains one global copy of learning rate $\lambda_t$ and radius $\sigma_t$ that is shared among all neurons}. They are designed to be dynamic, time-dependent functions and are shown in the $\Call{Train}{ }$ routine (which also depicts the full training step performed by our online SOM). We found, through empirical study, that the following set of equations 
yielded better clusters as opposed to original decay functions of \cite{kohonen1982self}:
\begin{align}
    \lambda_t = \lambda_0 (1 + t * \exp{(t / \tau_\lambda)})^{-1} \label{eq:lambda_t} \\
    \sigma_t = \sigma_0 (1 + t * \exp{(t / \tau_\sigma)})^{-1} \mbox{.} \label{eq:sigma_t}
\end{align}

\begin{figure}[!tb]
    \centering
    \includegraphics[width=\textwidth]{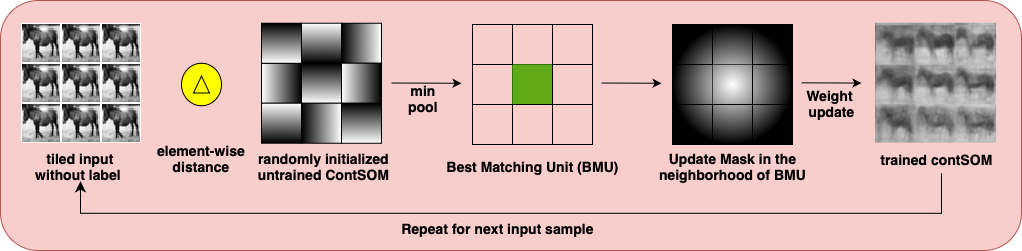}
    \caption{Illustration of the overall process for best-matching unit (BMU) selection and synaptic updating of the  CSOM for unsupervised online learning.}
    \label{fig:contsom_model}
\end{figure}

\begin{algorithm}[!tb]
    \caption{CSOM inference and learning processes.}
    \label{alg:cont_som}
    \begin{algorithmic}[1]
        \Require data sample $\mathbf{x}^{(k)}_i(t)$, SOM weight matrix $\mathbf{M}$, running variance matrix $\mathbf{M}^{\omega^2} \in \mathcal{R}^{D \times H}$, radius $\sigma_v, \sigma_V$, learning rate $\lambda_h$, BMU count for a unit ($\eta_u$ is count for unit $u$) 
        \State \textbf{Note:} $\text{HTile}(\mathbf{x},H)$ tiles vector $\mathbf{v}$ $H$ times horizontally; $\text{VTile}(\mathbf{x},H)$ tiles vector $\mathbf{x}$ $H$ times vertically, 
        \State \hspace{0.8cm} $d(a,b; \mathcal{G})$ is a distance over topology $\mathcal{G}$, where $a$ maps to $(a_i,a_j)$ (2D coordinate) and $b$ maps to $(b_i,b_j)$
        \Function{Update}{$\mathbf{X}_i, u, V, \sigma_u^{(t)}, \lambda_u^{(t)}, \mathbf{M}, \mathbf{M}^{\omega^2}, \mathcal{G}$}
        \State $\delta = (2 * \sigma_u^2)^{-1}$
        \State $\tau_1 = -\delta^{-1} * \log{(\frac{10^{-8}}{\lambda_u})}$
        \State $\mathbf{o}_{1,v_j} = 
                \begin{cases} 
                  1 & d(v_j,u; \mathcal{G}) < \sigma_u \\ 
                  0 & \text{otherwise} 
                \end{cases}$ \label{alg:binary_mask}
        \Comment{$v_j$ is any non-BMU neuron, $\mathbf{o} \in \{0,1\}^{1 \times H}$ (masking vector)} 
        \State $\mathbf{s}_{1,j} = 
                \begin{cases} 
                  \lambda_u & j = u \\ 
                  \lambda_v &  j = v_j 
                \end{cases}$ 
                \Comment{$\mathbf{s} \in \mathcal{R}^{1 \times H}$ learning rate vector}
        \State $\phi = \text{VTile}\Big( \{ \mathbf{o}_{v_j} \mathbf{s}_{v_j} \exp{[-d(u, j; \mathcal{G}) \delta]}, \; \text{for} \; j = 1,2,...,H \}, D \Big)$  \Comment{$\phi \in \mathcal{R}^{D \times H}$} 
        \State $\mathbf{M} = \mathbf{M} + \phi \odot (\mathbf{X}_i - \mathbf{M})$ \Comment{CSOM update step}
        \State $\eta_u = \eta_u + 1$
        \State $\lambda_{\omega} = \text{VTile}\Big( \{ (\lambda^0_{\omega} - 0.5) + (1 + \exp{(-d(u,j; \mathcal{G})/\tau_1)})^{-1}, \; \text{for} \; j = 1,2,...,H \}, D\Big)$ \Comment{$\lambda_\omega \in \mathcal{R}^{D \times H}$} \label{alg:init_lambda}
        \State $\lambda_\omega = \lambda_\omega \odot \mathbf{O} + (1 - \mathbf{O})$, where $\mathbf{O} = \text{VTile}(\mathbf{o},D)$  \Comment{scaling factor for updating running variance} 
        \State $\mathbf{M}^{\omega^2} = \lambda_\omega \mathbf{M}^{\omega^2} + (1 - \lambda_\omega) (\mathbf{X}_i - \mathbf{M})^2$ \Comment{update running variance of all neurons}
        \State $\sigma_u^{(t)} = \sigma_u^{(t-1)}  \exp{(\eta_u / \tau_{\sigma})}$ 
        \Comment{$\tau_{\sigma} \gets$ constant}
        \State $\lambda_u^{(t)} = \lambda_u^{(t-1)}  \exp{(\eta_u / \tau_\lambda)}$ 
        \Comment{$\tau_{\lambda} \gets$ constant}
        \EndFunction
        
        \Function{Train}{$\mathbf{x}^{(k)}_i(t), \mathbf{M}, \mathbf{M}^{\omega^2}, \mathcal{G}, t$}
        \State $\mathbf{X}_i \leftarrow \text{HTile}(\mathbf{x}^{(k)}_i(t), H)$ \Comment{tile/repeat $\mathbf{x}^{(t)}_i$ $H$ times horizontally to create $\mathbf{x}^{(k)}_i(t) = \mathbf{X}_i \in \mathcal{R}^{D \times H}$}
        \State $\mathbf{\Delta_\omega} = (\mathbf{X}_i - \mathbf{M})^2 / \mathbf{M}^{\omega}$   \Comment{distance calculation specific to CSOM , $\mathbf{M}^\omega = \sqrt{\mathbf{M}^{\omega^2}}$} \label{alg:csom_distance}
        \State $\mathbf{d} = \Delta_\omega[i,j]$ \Comment{$i \in \{1,2,..K\}$ and $j \in \{1,2,..L\}$}
        \State $u = \argmin_{h\in H} \mathbf{d}$  \Comment{BMU calculation}
        \State \Call{Update}{$X_i, u, V, \sigma_u^{(t)}, \lambda_u^{(t)}, \mathbf{M}, \mathbf{M}^{\omega^2}, \mathcal{G}$} \Comment{$V = H \setminus u $ (i.e., $V$ is set of all non-BMU indices)}
        \EndFunction
        \end{algorithmic}
\end{algorithm}

\subsection{The Continual Kohonen Map}
\label{sec:cont_som} 

In Algorithm \ref{alg:cont_som}, we present our proposed model, the continual SOM (CSOM \footnote{All the notations are summarized in Table \ref{tab:symbols}}, Figure \ref{fig:contsom_model}), built to process and generalize dynamically from samples drawn from a stream of tasks. Notice that, first of all, the CSOM now maintains an additional (non-negative) matrix $\mathbf{M}^{\omega^2} \in \mathcal{R}^{D \times H}_{+}$, which contains synaptic weight parameters associated with the ``running variance'' ($\omega^2$) of each neuronal unit in the system. This means that each unit/prototype in the CSOM is defined by two vectors of weights -- one for approximate unit means and another for approximate standard deviations.
The intuition is that each unit in our SOM model is generalized to maintain its own learnable multivariate Gaussian distribution (inspired by the latent variables of incremental Gaussian mixture models) with a diagonal covariance matrix. Notice that, in the $\Call{Update}{ }$ routine of our CSOM, the variance parameters are adjusted using a Hebbian-like rule inspired by Welford's online algorithm \cite{welford1962note} but modified to use the weighting provided by our CSOM's neighborhood function. We denote the initial scaling/update factor for running variance of every unit as $\lambda^0_{\omega}$ as shown in step \ref{alg:init_lambda} of Algorithm \ref{alg:cont_som}. In SOMs, the ratio of weight update that each connection goes through depends on its distance from the BMU. In CSOM, we empirically found that the update/scaling factor of running variance of every neuron is proportional to the magnitude of its weight update. Therefore, this update factor for a neuron is decayed as per its distance from the BMU to obtain $\lambda_{\omega}$ in step \ref{alg:init_lambda} of Algorithm \ref{alg:cont_som}. The running variance vector coupled to every neuron in the CSOM can be used to generate samples belonging to the class that matched it. Thus, the underlying structure of our neural system could be likened to a simple, dynamic generative model.

In addition to the local running variance parameters, the CSOM is designed to promote a form of neuronal competition driven by unit-centric learning and distance weighting parameters. \textbf{Specifically, each neuron $h \in H$ arranged in topology $\mathcal{G}$ is assigned an independently-controlled, dynamic radius $\sigma_h$ and learning rate $\lambda_h$ parameter (as shown in Figure \ref{fig:contsom_learning})}. As shown in Algorithm \ref{alg:cont_som}, in the routine $\Call{Update}{ }$, we particularly decay learning parameters only for the BMU ($u$), i.e., only $\sigma_u$ and $\lambda_u$ are decayed at time $t$. This localized decay is furthermore a function of the number of times that unit $u$ has been selected as the BMU, i.e., unit $u$ adjusts $\sigma_u$ and $\lambda_u$ as a function of its BMU count $\eta_u$. We do not decay $\sigma_u$ and $\lambda_u$ after they reach certain infinitesimally small threshold to ensure their values do not plummet to 0 and some positive weight update is always achieved as $t \rightarrow T$. The neighborhood function of the CSOM is also notably a function of the running variance parameters $\mathbf{M}^{\omega^2}$, further facilitating the calculation of a per-unit region of influence (treating each neuron as its own weighted multivariate Gaussian distribution). As a result, whenever a weight update is triggered, synaptic values are adjusted on a linear scale for the BMU $u$ while, for any non-BMU units $v_j$, the update is adjusted on a scale that decreases as the euclidean distance from $u$ increases. The CSOM utilizes a masking matrix ($\mathbf{o}$) to control the scale of its synaptic weight updates. The mask matrix also prevents any leaky weight updates caused by infinitesimally small values passed by the neighborhood function ($\phi$). 

Crucially, the proposed CSOM is a task-free model \cite{ClassIncrementalSurvey, Aljundi2018TaskFreeCL}, which means that it does not require any information about task boundaries (in the form of task descriptors). In addition, the CSOM, much like our reformulation of the SOM described earlier, is constructed to process data in an online, iterative fashion, adapting its per-unit parameters as a function of simulation time.

\begin{figure}[!tb]
    \centering
    \includegraphics[width=\textwidth]{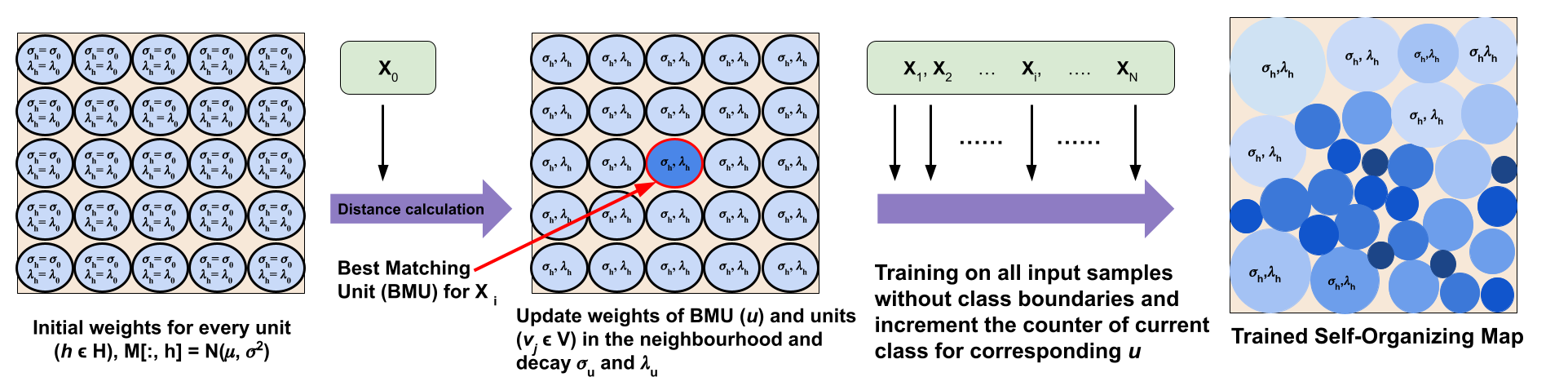}
    \caption{Unit centric parameter update in CSOM inspired from competitive learning}
    \label{fig:contsom_learning}
\end{figure}

\section{Mathematical Analysis}
In this section, we formally show that the CSOM reaches an equilibrium state or fixed point \cite{haag1974stability, stogin2024provably, burton2003stability}, meaning that it exhibits stability, a quality that improves its generalization ability as well as helps it avoid catastrophic forgetting.

Given the variance update rule in the continual self-organizing map (CSOM) described in the previous section, reproduce below as follows (for an arbitrary pattern vector sampled at time $t$):
\[ \omega_i^2(t+1) = \lambda_{\omega} \cdot \omega_i^2(t) + (1 - \lambda_{\omega}) \cdot (\mathbf{x}(t) - \mathbf{M}_i(t))^2, \]
where \(\lambda_{\omega} \in (0, 1)\) is the learning rate, we analyze the convergence of the series defining \(\omega_i^2(t)\).

\subsection*{Lemma: Convergence of the Variance Series}
\begin{lemma}
    The series for \(\omega_i^2(t)\) converges under the assumption that the sequence of squared deviations \({(\mathbf{x}(t) - \mathbf{M}_i(t))^2}\) is bounded.
\end{lemma} \label{lemma:conv}

\begin{proof}
Expanding the recurrence relation yields the following:
\[ \omega_i^2(t+1) = \lambda_{\omega}^t \cdot \omega_i^2(1) + (1 - \lambda_{\omega}) \sum_{k=0}^{t} \lambda_{\omega}^k \cdot \big(\mathbf{x}(t-k) - \mathbf{M}_i(t-k)\big)^2. \]

Assuming the sequence \(\big(\mathbf{x}(t) - \mathbf{M}_i(t)\big)^2\) is bounded by some value \(M > 0\), we analyze the convergence of the resultant geometric series:
\[ \sum_{k=0}^{\infty} \lambda_{\omega}^k \cdot (\mathbf{x}(t-k) - \mathbf{M}_i(t-k))^2. \]

Given \(0 < \lambda_{\omega} < 1\), the following series:
\[ \sum_{k=0}^{\infty} \lambda_{\omega}^k \]
is a convergent geometric series. By the comparison test, our series converges since each term is bounded by \(\lambda_{\omega}^k \cdot M\).

Hence, the series defining \(\omega_i^2(t)\) converges, particularly indicating that \(\omega_i^2(t)\) stabilizes as \(t \to \infty\), contributing to the SOM's capability to mitigate catastrophic forgetting by maintaining an adaptive and stable measure of variance.
\end{proof}

Next, we move to showing the convergence of the synaptic weight vectors of the proposed CSOM.
\begin{proposition}
For the weight vectors $\mathbf{M}_i(t)$ in a continual self-organizing map (CSOM) updated by the rule: 
\[ \mathbf{M}_i(t+1) = \mathbf{M}_i(t) + \lambda_i \cdot \phi(t, u, i) \cdot (\mathbf{x}(t) - \mathbf{M}_i(t)), \]
where $\lambda_i \in (0, 1)$ is a fixed learning rate, $u$ is the (integer) index of the best matching unit, and $\phi(t, u, i)$ is a neighborhood function\footnote{Note that, in Algorithm \ref{alg:cont_som}, we set this to be $\phi(t, u, i) = d(u, i)$.}, it holds that
\[ \|\mathbf{M}_i(t+1) - \mathbf{M}_i(t)\| \rightarrow 0 \]
as $t \rightarrow \infty$, indicating the convergence of the model's synaptic weight vectors.
\end{proposition} \label{prop:conv_wt}

\begin{proof}
The difference between successive weight vectors is given by:
\[ \Delta \mathbf{M}_i(t) = \mathbf{M}_i(t+1) - \mathbf{M}_i(t) = \lambda_i \cdot \phi(t, u, i) \cdot \big(\mathbf{x}(t) - \mathbf{M}_i(t)\big). \]

To demonstrate convergence, we need to show that $\|\Delta \mathbf{M}_i(t)\| \rightarrow 0$ as $t \rightarrow \infty$. 

1. \textbf{Bounding the Update Magnitude:} Notice that $\|\Delta \mathbf{M}_i(t)\|$ is bounded above by $\lambda_i \cdot \|\phi(t, i) \cdot (\mathbf{x}(t) - \mathbf{M}_i(t))\|$. Since $\lambda_i \in (0, 1)$ and assuming $\phi(t, u, i)$ is designed to decrease over time and as the distance from the BMU increases, then $\|\phi(t, u, i) \cdot (\mathbf{x}(t) - \mathbf{M}_i(t))\|$ also decreases.

2. \textbf{Decrease Over Time:} As $t$ increases, the influence of $\phi(t, u, i)$ diminishes, which in turn reduces the magnitude of $\Delta \mathbf{M}_i(t)$. This reduction is due to the adaptation of the weight vectors to the distribution of the input pattern vectors, leading to smaller corrections required over time as the neural map stabilizes.

3. \textbf{Limiting Behavior:} Given that the sequence of input vectors $\{\mathbf{x}(t)\}$ is bounded (a common assumption in many applications of the SOM), and the learning rate $\lambda_i$ is fixed or changes at steady rate ($\lambda_i >0$), the product $\lambda_i \cdot \phi(t, u, i)$ ensures that the updates to $\mathbf{M}_i(t)$ decrease in magnitude. Hence, $\|\Delta \mathbf{M}_i(t)\| \rightarrow 0$ as $t \rightarrow \infty$.

Therefore, the synaptic weight vectors $\mathbf{M}_i(t)$ converge as evidenced by the decrease in the magnitude of the updates between successive time steps, confirming the proposition.
\end{proof}

Next, we show that the CSOM reaches an equilibrium state as well as reaches a fixed point, indicating its continual adaptation of knowledge:
\begin{theorem}
Equilibrium in a continual self-organizing map (CSOM) is achieved through the calibration of the learning rate values $\lambda_{\omega}$ and $\lambda_i$, alongside the topological considerations enforced by the neighborhood radius $\sigma_u$ as well as the neighborhood function $\phi(t, u, i, \sigma_u)$. This calibrated mechanism facilitates the CSOM's adaptation to new input patterns while allowing it retain critical information about the distribution of past inputs, thus mitigating catastrophic forgetting.
\end{theorem} \label{thm:fixed_pt}

\begin{proof}
The proof of the above theorem focuses on the stabilization brought about by these parameters:

\textbf{1. Learning Rate Impact:}
The learning rates $\lambda_{\omega}$ and $\lambda_i$ directly influence the rate of adaptation for variance and weight vectors, respectively. A careful balance of these rates ensures that the CSOM network remains sensitive to new inputs without rapidly discarding historical data.

\textbf{2. Neighborhood Radius and Function:}
The neighborhood radius $\sigma_u$, which typically decreases over time, determines the extent to which the neighborhood surrounding the BMU is affected by each input pattern. The neighborhood function $\phi(t, u, i, \sigma_u)$ modulates the update magnitude based on $\sigma_u$ and the neuron's distance from the BMU, ensuring a cohesive topological adaptation across the network.

\textbf{Mathematical Justification:}
The equilibrium condition, characterized by diminishing updates to both variance and weight vectors, is achieved when $\|\omega_i^2(t+1) - \omega_i^2(t)\| \rightarrow 0$ and $\|\mathbf{M}_i(t+1) - \mathbf{M}_i(t)\| \rightarrow 0$ as $t \rightarrow \infty$. The decrease in $\sigma_u$ over time reduces the influence of new inputs on distant neurons, contributing to the network's overall stability.

\textbf{Conclusion:}
Through the interplay of $\lambda_{\omega}$, $\lambda_i$, $\sigma_u$, and $\phi(t, u, i, \sigma_u)$, the SOM achieves a dynamic equilibrium. This balance between sensitivity to new data and the retention of historical patterns allows the CSOM to adaptively map its input space (to neuronal unit space) while mitigating the risk of catastrophic forgetting.
\end{proof}

Finally, we formally show that the CSOM and its proposed synaptic update rule help the model to mitigate forgetting and learn new features without compromising prior knowledge:
\begin{theorem}
   Given a sequence of input vectors $\{\mathbf{x}(t)\}_{t=1}^\infty$ where $\mathbf{x}(t) \in \mathbb{R}^D$, and a continual self-organizing map (CSOM) with $H$ neurons, each neuron $i$ characterized by a weight vector $\mathbf{M}_i \in \mathbb{R}^D$ and a running variance $\omega_i^2 \in \mathbb{R}$, updated at each time $t$ by:

1. Variance Update Rule:
\[ \omega_i^2(t+1) = \lambda_{\omega} \cdot \omega_i^2(t) + (1 - \lambda_{\omega}) \cdot (\mathbf{x}(t) - \mathbf{M}_i(t))^2, \]

2. Synaptic Weight Update Rule:
\[ \mathbf{M}_i(t+1) = \mathbf{M}_i(t) + \lambda_i \cdot \phi(t, u, i) \cdot (\mathbf{x}(t) - \mathbf{M}_i(t)), \]

with $\lambda_{\omega}, \lambda_i \in (0, 1)$ as learning rates, and $\phi(t, u, i)$ as the neighborhood function, then the SOM achieves a dynamic equilibrium, enhancing its stability and mitigating catastrophic forgetting.
\end{theorem}
\begin{proof}
    The proof consists of three key parts:

\textbf{Part 1: Variance Stability.} The exponential decay factor $(1-\lambda_{\omega})$ ensures that older inputs have a diminishing influence on the running variance $\omega_i^2(t)$. For large $t$, this stabilizes, indicating the SOM's ability to incorporate new variance information without losing the significance of past data variance. Mathematically, this can be shown by analyzing the series convergence of $\omega_i^2(t)$, which was demonstrated in Lemma \ref{lemma:conv}

\textbf{Part 2: Convergence of Synaptic Weight Vectors.} The weight vectors $\mathbf{M}_i(t)$ are updated towards (the direction of) new inputs with a learning rate that is modulated by $\phi(t, i)$, ensuring that neurons closer to the BMU have a higher rate of adaptation. This promotes the stability and convergence of $\mathbf{M}_i(t)$, which is crucial for maintaining knowledge of learned patterns over time. The convergence can be argued by demonstrating that the updates lead to a decrease in the difference between successive weight vectors, $\|\mathbf{M}_i(t+1) - \mathbf{M}_i(t)\| \rightarrow 0$ as $t \rightarrow \infty$, as was formally proven in Proposition \ref{prop:conv_wt}.

\textbf{Part 3: Dynamic Equilibrium and the Mitigation of Catastrophic Forgetting.} Combining the effects of the variance and synaptic weight updates, the CSOM maintains a balance between adapting to new (sensory) data and preserving information about the distribution of the past inputs it has encountered. This equilibrium is achieved through the careful calibration of $\lambda_{\omega}$ and $\lambda_i$, alongside the topological considerations enforced by $\phi(t, u, i)$, as shown in Theorem \ref{thm:fixed_pt}. 

Thus, this mathematical framework ensures that the CSOM's neuronal units' updates are such that neither new nor old information is disproportionately favored, which mitigates catastrophic forgetting. 
Therefore, by the specified update rules and their mathematical properties, the CSOM achieves a dynamic equilibrium that effectively mitigates catastrophic forgetting.
\end{proof}

\section{Experiments}
\label{sec:results}
To test our proposed CSOM and determine its performance on competitive datasets, we performed experiments across three databases under two continual learning setups. 

\subsection{Datasets}
\label{ssec:datasets}
To evaluate the CSOM, we employed three grayscale datasets along with one containing natural image samples. 

\paragraph{Split MNIST, Fashion MNIST, and KMNIST:}
We used variations of MNIST to test our neural models -- specifically, the original MNIST database \cite{mnist}, Fashion MNIST (FMNIST) \cite{fmnist}, and Kuzushiji-MNIST (KMNIST) \cite{kmnist}, all containing $28 \times 28$ gray-scale pixel images. 
Furthermore, we transformed these datasets by normalizing them, i.e., diving them by $255.0$, which brought them into the value range of $[0.0,1.0]$, thus making the learning process easier for the SOMs/CSOM models. 

\paragraph{Split CIFAR-10:}

To test the CSOM on images belonging to different (more complex) distributions, we also utilized the CIFAR-10 dataset \cite{Krizhevsky2009LearningML}, which contains a large number of samples with $32 \times 32$ Red-Green-Blue (RGB) images. 
We specifically converted the CIFAR-10 images to grayscale to test the neural models of this study (and defer the use of separate color channels for future work). Furthermore, we normalized the images as in the MNIST datasets, bringing their pixel values into the range of $[0.0,1.0]$.

\begin{figure}[!t]
    \centering
    \begin{subfigure}{0.32\textwidth}
        \centering
        \includegraphics[width=\linewidth]{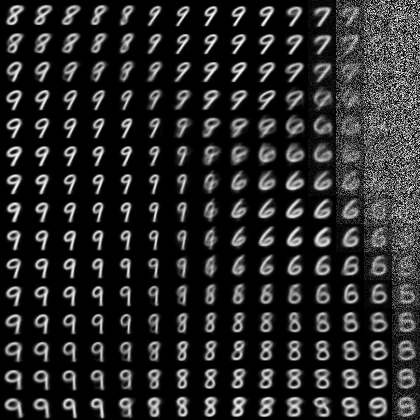}
        \label{fig:0}
    \end{subfigure}
    \begin{subfigure}{0.32\textwidth}
        \centering
        \includegraphics[width=\linewidth]{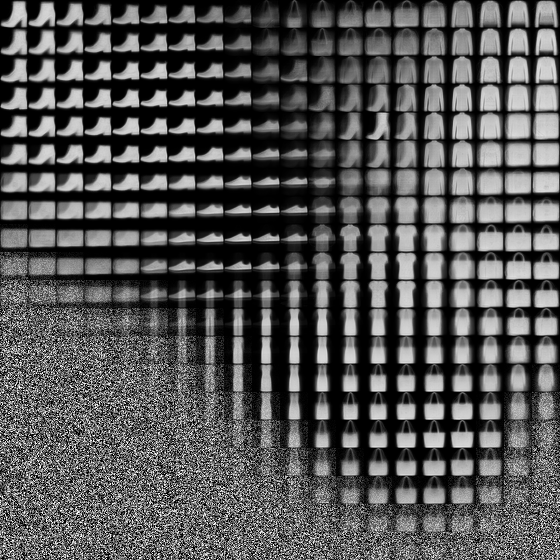}
        \label{fig:1}
    \end{subfigure}
    \begin{subfigure}{0.32\textwidth}
        \centering
        \includegraphics[width=\linewidth]{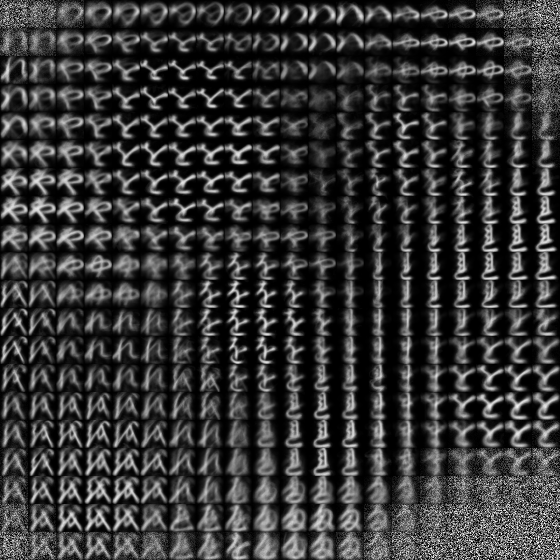}
        \label{fig:2}
    \end{subfigure} \\
    \vspace{-0.2cm}
    \begin{subfigure}{0.32\textwidth}
        \centering
        \includegraphics[width=\linewidth]{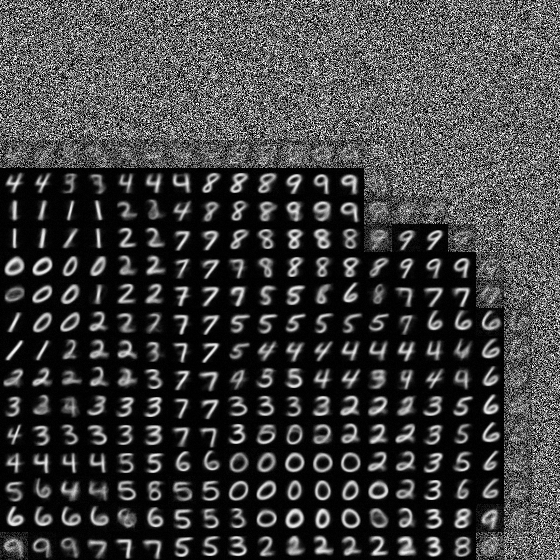}
        \label{fig:4}
    \end{subfigure}
    \begin{subfigure}{0.32\textwidth}
        \centering
        \includegraphics[width=\linewidth]{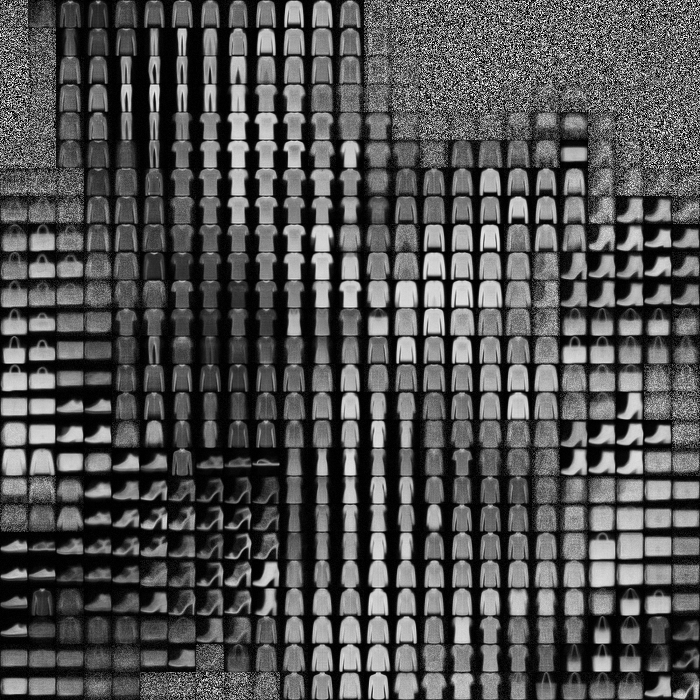}
        \label{fig:5}
    \end{subfigure}
    \begin{subfigure}{0.32\textwidth}
        \centering
        \includegraphics[width=\linewidth]{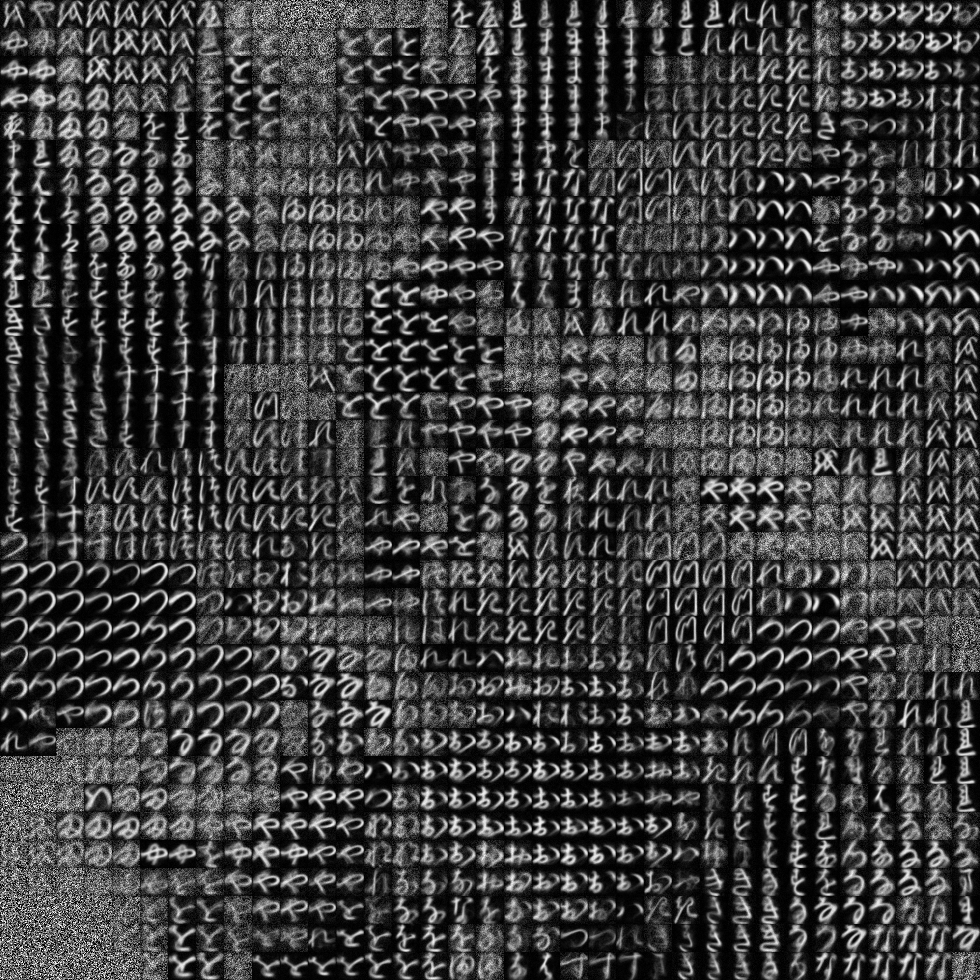}
        \label{fig:6}
    \end{subfigure}
    \caption{Class-incrementally adapted classical SOM (top row) versus continual SOM (bottom row) on: MNIST (Left), FashionMNIST (Middle), and KMNIST (Right). 
    }
    \label{fig:contSOM_results_incremental}
\end{figure}

\begin{figure}[!t]
    \centering
    \begin{subfigure}{0.32\textwidth}
        \centering
        \includegraphics[width=\linewidth]{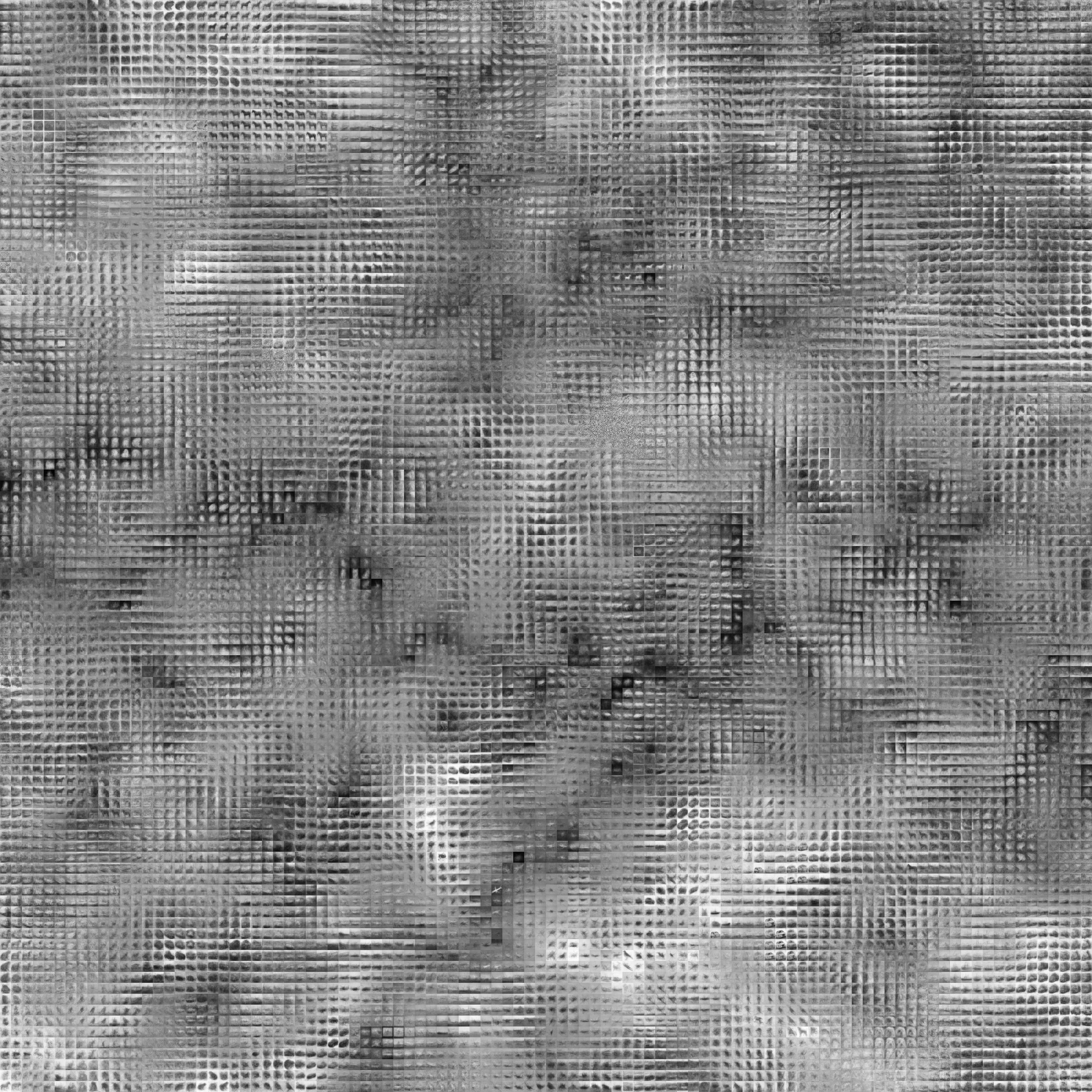}
    \end{subfigure}
    \begin{subfigure}{0.32\textwidth}
        \centering
        \includegraphics[width=\linewidth]{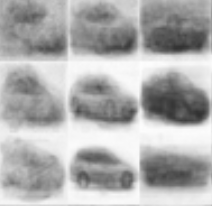}
    \end{subfigure}
    \begin{subfigure}{0.32\textwidth}
        \centering
        \includegraphics[width=\linewidth]{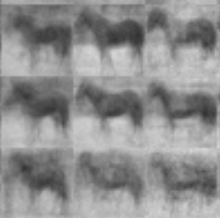}
    \end{subfigure}
    \caption{(Left) CSOM containing 100x100 neurons/units trained class incrementally on grayscale images of CIFAR-10. (Right) Two snapshots of trained CSOM clusters}
    \label{fig:cifar_som}
\end{figure}

\paragraph{Model Baselines:}
We adapt the Dendritic SOM (DendSOM) \cite{dendSOM}, further utilizing their label prediction method to evaluate the performance of our proposed model. The DendSOM uses a hit matrix to determine the label of a trained unit in the evaluation phase. We maintain a similar hit matrix (BMU count, $\eta_u$) later used in the evaluation phase for label prediction. This ultimately helps us to establish a fair comparison between our model and baselines.

\subsection{Learning Simulation Setups}
\label{sec:sim_setups}

\noindent
\textbf{Class Incremental Learning: }In this study, we focus on the continual learning setup where tasks consist of data points belonging entirely to one single class. Specifically, this means that for any given task $\mathcal{T}_t$, its training dataset $\mathcal{D}^{(t)}_{train}$ contains all data points labeled under one specific class $c \in C$ (where $C$ is the total number of unique classes in the entire benchmark dataset). For instance, for the case of Split-MNIST, this means that task $\mathcal{T}_0$ consists entirely of patterns with the class label of the digit zero, and task $\mathcal{T}_1$ consists entirely of the digit one, and so on and so forth. Furthermore, note that all models and baselines studied are trained under the same setup with the same seeding/shuffling of data points in order to ensure a fair comparison. All models are trained to process the data points incrementally/online, one sample at a time, and no data points are ever revisited (meaning that each model is only allowed one single epoch or pass through an entire task data subset, effectively simulating the online streaming learning setting \cite{Agarwal2008KernelbasedOM, 10.1609/aaai.v33i01.33013232}).

\noindent
\textbf{Domain Incremental Learning: } In this setup we simulated a domain incremental learning within a given dataset. This is similar to a task incremental setup; however, here, we divided the datasets into five tasks, where each task had data samples representing two digits, labeled in binary, i.e., $c \in [0,1]$ ($0$ encoded class one, $1$ encoded class two).

\subsection{Training and Architecture Parameters}
We followed a square design for the topology underlying all of the models examined in this work. This means that every SOM/CSOM had $K \times L$ units in its topology, ($K = L$). 
In Table \ref{tab:params}, the column for ($K = L$) indicates values of $K$ for MNIST, Fashion-MNIST, KMNIST, respectively. The initial value of the running variance for all pixels in every SOM unit ($\omega^2_0$) was decided based on the distribution of pixels in the input dataset. We found empirically that setting $\omega^2_0$ slightly higher than the variance of input distribution gave us a good starting point. Similarly, setting $\lambda^0_\omega$ within the range $[0.9, 0.99]$ helped to maintain an effective balance between the old and new values of the running variance at every simulation step. For the DendSOM, we modeled/instantiated four copies of SOM ($M$) where the unit topology size in each copy was [$14 \times 14$]. Since the DendSOM breaks the input image into patches for every copy of SOM, the easiest way to divide an MNIST-type of image having dimensions [$28 \times 28$] would be to have its four patches be of shape [$14 \times 14$]. We followed this for the implementations of DendSOM presented in our work. Through a brute force approach, we found that setting $\tau_{\sigma} = 8$ and $\tau_{\lambda} = 45$ performed a gradual decay of $\sigma$ and $\lambda$ throughout training.

\begin{table}[!t]
    \centering
    \begin{tabular}{cccccc}
        \hline
        Model & ($K = L$) & $\sigma_0$ & $\lambda_0$ & $\omega^2_0$ & $\lambda^0_{\omega}$ \\
        \hline
        Vanilla SOM & [15, 20, 20] & 0.6 & 0.07 & 1 & 0 \\
        \hline
        DendSOM & [10, 10, 10] & 1.5 & 0.07 & 0 & 0 \\
        \hline
        CSOM & [15/20, 25, 35] & 1.5 & 0.07 & 0.5 & 0.9 \\
        \hline
    \end{tabular}
    \caption{Parameter setting for the architectures and their training}
    \label{tab:params}
\end{table}

\subsection{Evaluation}
\label{ssec:evaluation}

For model evaluation, we adapted the point-wise mutual information (PMI) \cite{dendSOM} measure, which loosely follows the Hebbian rule of learning to suit the design logic of CSOM. Formally, the PMI is calculated as follows:
\begin{align}
    PMI(l; BMU) = \log{\frac{P(l|BMU)}{P(l)}} \label{eq:pmi} \\
    P(l|BMU) = \frac{\eta[l, BMU]}{\Sigma_{i \in Labels} \eta[i, BMU]} \label{eq:pmi_numerator}\\
    P(l) = \frac{\Sigma_{h \in Units}\eta[l, h]}{\Sigma_{h \in Units} \Sigma_{i \in Labels} \eta[i, h]} \label{eq:pmi_denominator} \\
    Predicted\ Label = \argmax_{l \in Labels} PMI(l; BMU). \label{eq:prediction}
\end{align}
The PMI of two entities could take on either a positive or negative value, depending on whether they co-occur more frequently or less frequently compared to an independence assumption. The PMI measure can become zero if the two entities are independent of each other. This PMI measure is used to predict the label of a trained unit based on a hit matrix (i.e. BMU count, $\eta_u$) which stores the number of times that a unit was selected as the BMU for any input sample. The PMI uses label information from the input samples to maintain the hit matrix that is used only at the test time for accuracy calculation; as a result, the CSOM is unsupervised in training even though its test-time performance is evaluated through a supervised lens. 

Note that we measure the average accuracy of SOM models based on whether the predicted SOM/CSOM unit's label matches the input sample's expected label. Unlike the L2 distance with running variance used at training time, we employed the cosine similarity at test time in order to measure the BMU for a test input sample and to determine its predicted label. Using this method, we calculated distinct accuracy measures from test samples for each task and stored these in what is known in the continual learning literature as the accuracy task matrix (of shape $T \times T$ for $T$ tasks). This helped us to calculate classical lifelong learning metrics such as average accuracy (ACC) \cite{FastandSlow}, backward transfer (BWT) \cite{GEM}, as well as the forgetting measure (FM) and learning accuracy (LA), both of which were mentioned in \cite{YinYL21}.  Formally, these evaluation metrics are defined as follows:
\begin{align}
    \text{ACC}\textcolor{green}{(\uparrow)} & = \frac{1}{T} \Sigma_{i=1}^{T} a_{i, T}   && \text{(higher is better)} \label{eq:acc} \\
    \text{BWT} \textcolor{blue}{(\downarrow)} & = \frac{1}{T-1} \Sigma_{i=1}^{T-1} a_{T, i} - a_{i,i} && \text{(lower is better)}\label{eq:bwt} \\
    \text{FM} \textcolor{blue}{(\downarrow)} & = \frac{1}{T} \Sigma_{i=1}^{T} \lvert a_{i, T} - a_{i}^{*} \rvert && \text{(lower is better)} \label{eq:fm} \\
    \text{LA} \textcolor{green}{(\uparrow)} & = \frac{1}{T} \Sigma_{i=1}^T a_{i,i} && \text{(higher is better)}. \label{eq:la}
\end{align}
In the above equations, $a_{i,j}$ is the accuracy for task $i$ after training on task $j$. The BWT measure helps us quantify how training on a new task $\mathcal{T}_k$ affects the accuracy of CSOM on the previously trained task $\mathcal{T}_{k-1}$. In other words, the BWT quantifies the amount of forgetting for a task as the training progresses incrementally for new tasks in time for any given algorithm.
FM is the measure of the difference between a model's final performance and its best performance ($a_i^{*}$) for a task. LA measures the average of the current task accuracy of CSOM.

\subsection{Simulation Results}
\label{sec:simulation_results}

In Table \ref{tab:results}, we present our full experimental results for the three grayscale datasets. In addition, we present quantitative benchmark measurements on the CIFAR-10 dataset in Table \ref{tab:cifar}, comparing to several prior, performant continual learning neural models.

\begin{table*}[!t]
    \centering
        \begin{subtable}{1\textwidth}
        \centering
        \begin{tabular}{rrrrr}
        \hline
        
        MNIST & ACC ($\uparrow$) & LA ($\uparrow$) & FM ($\downarrow$) & BWT ($\downarrow$) \\
        \hline
        vanilla SOM & 22.89 $\pm$ 0.02 & 46.4 $\pm$ 0.01 & 34.99 $\pm$ 0.01 & -26 $\pm$ 0.04 \\
        \hline
        DendSOM & 17.13 $\pm$ 0.01 & 40.34 $\pm$ 0.03 & 23.33 $\pm$ 0.03 & -25.78 $\pm$ 0.04 \\
        \hline
        \textbf{CSOM} & \textbf{85.03 $\pm$ 4.28} & \textbf{92.26 $\pm$ 2.49} & \textbf{7.23 $\pm$ 2.42} & \textbf{-8.02 $\pm$ 2.7} \\
        \hline
        \end{tabular}
        
        \begin{tabular}{rrrrr}
        \hline
        FMNIST & ACC ($\uparrow$) & LA ($\uparrow$) & FM ($\downarrow$) & BWT ($\downarrow$) \\
        \hline
        vanilla SOM & 32.07 $\pm$ 0.03 & 57.52 $\pm$ 0.03 & 37.01 $\pm$ 0.03 & -28.27 $\pm$ 0.02 \\
        \hline
        DendSOM & 16.62 $\pm$ 0.03 & 27.36 $\pm$ 0.03 & \textbf{11.34 $\pm$ 0.03} & \textbf{-11.92 $\pm$ 0.03} \\
        \hline
        \textbf{CSOM} & \textbf{75.13 $\pm$ 2.68} & \textbf{88.23 $\pm$ 0.56} & 13.11 $\pm$ 2.93 & -14.56 $\pm$ 3.25 \\
        \hline
        \end{tabular}
        
        \begin{tabular}{rrrrr}
        \hline
        KMNIST & ACC ($\uparrow$) & LA ($\uparrow$) & FM ($\downarrow$) & BWT ($\downarrow$) \\
        \hline
        vanilla SOM & 23.65 $\pm$ 0.03 & 54 $\pm$ 0.05 & 35.28 $\pm$ 0.01 & -33.79 $\pm$ 0.02 \\
        \hline
        DendSOM & 10.95 $\pm$ 0.01 & 12.92 $\pm$ 0.01 & \textbf{2.02 $\pm$ 0} & \textbf{-2.18 $\pm$ 0.01} \\
        \hline
        \textbf{CSOM} &  \textbf{81.602 $\pm$ 2.82} & \textbf{88.61 $\pm$ 1.6} & 7.04 $\pm$ 1.54 & -7.78 $\pm$ 1.73 \\
        \hline
        \end{tabular}
        \caption{Unsupervised online (epochs = 1) class incremental training}
        \label{tab:class_incremental}
    \end{subtable}
    \bigskip
    \begin{subtable}{1\textwidth}
        \centering
        \begin{tabular}{rrrrr}
        \hline
        MNIST & ACC ($\uparrow$) & LA ($\uparrow$) & FM ($\downarrow$) & BWT ($\downarrow$)\\
        \hline
        vanilla SOM & 55.03 $\pm$ 0.01 & 80.07 $\pm$ 0.02 & 25.59 $\pm$ 0.02 & -30.71 $\pm$ 0.04 \\
        \hline
        DendSOM & 58.40 $\pm$ 0.03 & 61.37 $\pm$ 0.05 & 7.30 $\pm$ 0.02 & \textbf{-3.7 $\pm$ 0.04} \\
        \hline
        \textbf{CSOM} & \textbf{89.09 $\pm$ 3.32} & \textbf{93.02 $\pm$ 2.03} & \textbf{4.03 $\pm$ 1.71} & -4.92 $\pm$ 2.11 \\
        \hline
        \end{tabular}
        
        \begin{tabular}{rrrrr}
        \hline
        FMNIST & ACC ($\uparrow$) & LA ($\uparrow$) & FM ($\downarrow$) & BWT ($\downarrow$) \\
        \hline
        vanilla SOM & 79.36 $\pm$ 0.06 & 89.45 $\pm$ 0.01 & 12.83 $\pm$ 0.06 & -12.61 $\pm$ 0.06 \\
        \hline
        DendSOM & 50.11 $\pm$ 0 & 53.13 $\pm$ 0.03 & 3.98 $\pm$ 0.03 & -3.76 $\pm$ 0.04 \\
        \hline
        \textbf{CSOM} & \textbf{95.61 $\pm$ 0.92} & \textbf{97.14 $\pm$ 0.4} & \textbf{1.57 $\pm$ 0.89} & \textbf{-1.92 $\pm$ 1.12} \\
        \hline
        \end{tabular}
        \begin{tabular}{rrrrr}
        \hline
        KMNIST & ACC ($\uparrow$) & LA ($\uparrow$) & FM ($\downarrow$) & BWT ($\downarrow$) \\
        \hline
        vanilla SOM & 57.18 $\pm$ 0.01 & 78.91 $\pm$ 0.01 & 21.72 $\pm$ 0.01 & -27.16 $\pm$ 0.01 \\
        \hline
        DendSOM & 52.3 $\pm$ 0.02 & 54.55 $\pm$ 0.03 & \textbf{3.62 $\pm$ 0.01} & \textbf{-2.8 $\pm$ 0.02} \\
        \hline
        \textbf{CSOM} & \textbf{88.11 $\pm$ 1.93} & \textbf{91.97 $\pm$ 1.04} & 3.88 $\pm$ 1.04 & -4.83 $\pm$ 1.27 \\
        \hline
        \end{tabular}
        \caption{Unsupervised online (epochs = 1) domain incremental learning}
        \label{tab:domain_incremental}
    \end{subtable}
    \caption{Summary of (n=10 trials; mean and standard deviation of scores) results for class-incremental and domain incremental variant of vanilla SOM, DendSOM, and contSOM. For DendSOM, we created $n=4$ soms with unit size = 14 }
    \label{tab:results}
\end{table*}


\begin{table}[!t]
    \centering
    \begin{tabular}{lrrrr}
        \hline
        Cifar-10 Model & ACC ($\uparrow$) & LA ($\uparrow$) & FM ($\downarrow$) & BWT ($\downarrow$) \\
        \hline
        Finetune & 10 $\pm$ 0 & 37.49 $\pm$ 0 & 27.83 $\pm$ 0 & -30.54 $\pm$ 0 \\
        Replay & 14.84 $\pm$ 1.7E-15 & 87.86 $\pm$ 1.42E-14 & 73.64 $\pm$ 0 & -81.13 $\pm$ 0 \\
        LwF  & 9.87 $\pm$ 0.17 & 11.37 $\pm$ 1.04 & \textbf{1.56 $\pm$ 1.31} & \textbf{-1.66 $\pm$ 1.36} \\
        iCarl & 10 $\pm$ 0 & 90.65 $\pm$ 1.42E-14 & 80.65 $\pm$ 0 & -89.61 $\pm$ 1.42E-14 \\
        EWC & 10.11 $\pm$ 0 & 43.93 $\pm$ 7.1E-15 & 35.69 $\pm$ 0 & -37.57 $\pm$ 0 \\
        Memo & 11.42 $\pm$ 0.48 & 96.59 $\pm$ 2.22 & 85.17 $\pm$ 1.74 & -94.64 $\pm$ 1.94 \\
        Podnet & 12.55 $\pm$ 1.7E-15 & 77.47 $\pm$ 0 & 64.92 $\pm$ 0 & -72.13 $\pm$ 0 \\
        BiC & 12.86 $\pm$ 0.72 & \textbf{96.77 $\pm$ 0.99} & 83.91 $\pm$ 1 & -93.24 $\pm$ 1.11 \\
        SCALE & 29.54 $\pm$ 0.6 & 47.31 $\pm$ 3.21 & 20.23 $\pm$ 2.11 & -19.75 $\pm$ 3.52 \\
        \textbf{CSOM} & \textbf{31.3} $\pm$ 0.3 & 46.14 $\pm$ 0.48 & 14.92 $\pm$ 0.24 & -16.49 $\pm$ 0.25 \\
        \hline
    \end{tabular}
    \caption{Summary (trials=10) of online (epochs=1) class incremental training on split-Cifar10}
    \label{tab:cifar}
\end{table}

\begin{table}[!hbt]
    \centering
    \begin{tabular}{lcccccc}
        \hline
        Papers & convolution & buffer & Data & single & class & task \\
          &   & storage & Augmentation & pass & labels & boundaries \\
        \hline
        Finetune & \textcolor{green}{\checkmark} & \textcolor{red}{\xmark} & \textcolor{red}{\xmark} & \textcolor{red}{\xmark} & \textcolor{green}{\checkmark} & \textcolor{red}{\xmark} \\
        Replay & \textcolor{green}{\checkmark} & \textcolor{green}{\checkmark} & \textcolor{red}{\xmark} & \textcolor{red}{\xmark} & \textcolor{green}{\checkmark} & \textcolor{green}{\checkmark} \\
        LwF \cite{8107520} & \textcolor{green}{\checkmark} & \textcolor{red}{\xmark} & \textcolor{red}{\xmark} & \textcolor{red}{\xmark} & \textcolor{green}{\checkmark} & \textcolor{green}{\checkmark} \\
        iCarl \cite{8100070} & \textcolor{green}{\checkmark} & \textcolor{green}{\checkmark} & \textcolor{red}{\xmark} & \textcolor{red}{\xmark} & \textcolor{green}{\checkmark} & \textcolor{green}{\checkmark} \\
        EWC \cite{kirkpatrick2017overcoming} & \textcolor{red}{\xmark} & \textcolor{red}{\xmark} & \textcolor{red}{\xmark} & \textcolor{red}{\xmark} & \textcolor{green}{\checkmark} & \textcolor{green}{\checkmark} \\
        Memo \cite{zhou2023a} & \textcolor{green}{\checkmark} & \textcolor{green}{\checkmark} & \textcolor{red}{\xmark} & \textcolor{red}{\xmark} & \textcolor{green}{\checkmark} & \textcolor{green}{\checkmark} \\
        Podnet \cite{podnet} & \{\textcolor{green}{\checkmark}, \textcolor{red}{\xmark}\} & \textcolor{green}{\checkmark} & \textcolor{red}{\xmark} & \textcolor{red}{\xmark} & \textcolor{green}{\checkmark} & \textcolor{green}{\checkmark}\\
        BiC \cite{Wu_2019_CVPR} & \textcolor{green}{\checkmark} & \textcolor{green}{\checkmark} & \textcolor{red}{\xmark} & \textcolor{red}{\xmark} & \textcolor{green}{\checkmark} & \textcolor{green}{\checkmark} \\
        SCALE \cite{SCALE} & \textcolor{green}{\checkmark} & \textcolor{green}{\checkmark} & \textcolor{green}{\checkmark} & \textcolor{green}{\checkmark} & \textcolor{red}{\xmark} & \textcolor{red}{\xmark} \\
        CSOM & \textcolor{red}{\xmark} & \textcolor{red}{\xmark} & \textcolor{red}{\xmark} & \textcolor{green}{\checkmark} & \textcolor{red}{\xmark} & \textcolor{red}{\xmark} \\
        \hline
    \end{tabular}
    \caption{Training schemes followed by all approaches}
    \label{tab:benchmarks}
\end{table}\vspace{-2.5mm}

\paragraph{The Classical SOM: } As observed in Table \ref{tab:results}, the standard/classical SOM, or ``Vanilla SOM'' (which indicates that no special task-driven mechanism was integrated, thus meaning that we used the online model presented in Algorithm \ref{alg:kohonen_map}), 
observably forgets quite strongly across all three benchmarks/datasets. In short, our results empirically confirm that, indeed, the standard SOM forgets information despite the potential (for reduced neural cross-talk) offered by its competitive internal activities. 
We present visual samples of the memories acquired by our SOM model(s) in Figure \ref{fig:contSOM_results_incremental}, which qualitatively demonstrates the types of memories these specific models acquire during learning -- 
for datasets such as Split-MNIST, the SOM only has stored in its internal memories prototypes/templates that match the last few digits/tasks of the underlying task sequence, i.e., largely the digit ``8'' and ``9''.

\paragraph{The CSOM:} \label{results_csom} 
The conventional distance metrics, such as euclidean distance or cosine similarity, often identify one of the previously trained units as the BMU for the current incoming task samples due to the semantic overlap present among different task samples in datasets like MNIST (as explained in section \ref{ssec:datasets}). This means that the images from different classes in MNIST-like datasets appear to be quite similar due to the common black background and white foreground. This increases the probability of an already trained SOM unit to be selected as BMU for an input sample from new class instead of an untrained unit being selected as BMU. Thus, a vanilla SOM fails to use its full capacity and ultimately does not train on all the available units in a continual learning setting.
The distance function shown in the $\Call{Train}{ }$ subroutine of algorithm \ref{alg:cont_som}, divides the squared terms of the L2 distance by the running standard deviation of units obtained from $\mathbf{M}^\sigma$. This helps to preserve the trained units' memories (with respect to previously encountered classes) and successfully allocates the untrained units in the SOM to new incoming tasks. In other words, such a distance metric for identifying BMU helps to balance the stability and plasticity of the neural system. More importantly, it eliminates the need for a task vector (refer to Table \ref{tab:benchmarks}), which sets up context for the incoming tasks in most modern-day continual learning models. As a result, the CSOM can learn in a streaming manner without needing task boundaries which makes it a powerful, unsupervised lifelong learning system.

Experimentally, we noticed that the Gaussian function used to scale the synaptic weight updates in a neighbourhood 
was liable to cause leaky weight updates in the units that were very far away from the BMU in topology. This caused an explosion of weight update in the available untrained units, thus undesirably increasing the size of cluster for current incoming task. To control this behaviour, we enforced a hard bound over the synaptic weight update in the neighbourhood using the current $\sigma_u$ value of the BMU. For this we create a binary mask as shown in step \ref{alg:binary_mask} of $\Call{Update}{ }$ subroutine in Algorithm \ref{alg:cont_som}. 

In Table \ref{tab:cifar}, we present a comparison of model performance measurements on the split-CIFAR10 benchmark. We used PyCIL \cite{zhou2023pycil} to obtain all the benchmarks for standard continual learning models except SCALE \cite{SCALE}. For them, we kept the same setting of total memory size=$2000$ and memory per class=$20$. These were the standard parameter settings that \cite{zhou2023pycil} used to obtain their baseline results. We conducted $10$ experimental trials for all the benchmarks and, after every trial, we constructed a task matrix in order to calculate the ACC, LA, FM, and BWT values for any model. Among all the benchmarks, \cite{SCALE} had the exact same setting (shown in Table \ref{tab:benchmarks}) as ours and had state-of-the-art (SOTA) ACC values in a semi-supervised learning setting. Although it had marginally higher mean LA value, their standard deviation was almost $7$ times higher than our CSOM after $10$ trials. Except for EWC \cite{kirkpatrick2017overcoming}, all of the models used convolution and most of them used a storage buffer (memory). Nevertheless, our CSOM model yielded the best average accuracy among all others without maintaining any buffer storage or requiring convolution to obtain better performance, despite operating in an online setting (epochs=$1$, each input sample processed once).

As shown in Figure \ref{fig:contSOM_results_incremental}, 
as a result of the neural clusters formed, our model appeared to acquire a good latent representation of every input task. This means that there is a good amount of  variance among the trained units allocated for every task. However, despite our good experimental results, we could not obtain an equal number of trained units per class. Moreover, some neuronal units (especially those at the line of separability of clusters) have fuzzy representations. This may occur if a neuronal unit receives a weight update that corresponds to multiple different classes.

Nevertheless, as indicated in Table \ref{tab:results} and \ref{tab:cifar}, 
\textbf{the CSOM achieves the best performance compared to all of the other models/variants of SOM on all four benchmarks/datasets and successfully beats the SOTA model on split-CIFAR10 dataset}. This empirical result indicates that our CSOM framework, leveraging a competitive learning scheme, can yield the potentially best memory retention (or greatest reduction in forgetting) when processing classes incrementally from a data stream. The visualization of the samples synthesized for this particular model -- see Figure 
\ref{fig:contSOM_results_incremental} -- 
also corroborate this result qualitatively; the samples, in this case, look the clearest and the model's internal units seem to represent most of the individual classes/tasks (including those presented at the start of the task sequence).

\noindent
\textbf{Discussion:} The nature of the distance function and the hyperparameter decay method used in all the SOM models induces an inherent bias towards the initially encountered classes in the class incremental setting. Specifically, it assigns a larger cluster or more trained units in the SOM to initial classes compared to classes that appear later in the incremental order. Although the proposed CSOM produced the best results out of all of the model variants, further improvement is certainly possible in order to obtain clearer/crisper representations of input samples in a trained model final output of a trained SOM (this will be the subject of future work). By augmenting the SOM's internal dynamics with additional synaptic parameters that maintained the running variance for all of its neuronal units, we show that the CSOM can further be used as a useful generative model, facilitating a natural internal memory replay mechanism that could further aid in memory retention in more complex continual learning setups. Notably, this could be an important component to improve the performance of memory-augmented neural networks \cite{das1992using, graves2014neural, joulin2015inferring, mali2020neural, stogin2024provably}.

\section{Conclusions}
\label{sec:conclusion}

In this paper, we investigated catastrophic forgetting -- a phenomenon where a neural-based intelligent agent forgets the information/knowledge that it acquired on previous datasets/tasks whenever it starts processing a new dataset/task -- in the context of continual unsupervised learning utilizing the classical self-organizing map (SOM). Specifically, we designed an adaptation of the classical model to this setting and empirically found that it exhibited severe interference and low memory retention. In light of this, we proposed a novel generalization of the model -- the continual SOM (CSOM) -- which we theoretically and experimentally demonstrated, across several data benchmarks, that it exhibited vastly improved memory retention ability notably through the use of running variance and decay mechanisms, locally embedded into each neuronal unit. Future work will include an examination of the integration of the cross-task memory retention ability of our model into supervised and semi-supervised neural systems on additional, larger-scale data problems, as well as further improving the quality of the CSOM's internally acquired neural prototypes through additional extensions/generalizations of its local unit-variance parameters.

\bibliographystyle{acm}
\bibliography{ref}

\clearpage

\section{Appendix}

\textbf{Revisiting the hyperparameters defined:} \\
Table \ref{tab:symbols} indicates key notations, symbols and abbreviations used in this paper.

\begin{table}[!htb]
    \centering
    \begin{tabular}{ll}
        \hline
        Item & Explanation \\
        \hline
        $\mathcal{G} = (K \times L)$ & Network Topology (dimensions of the SOM) \\
          &  \textbf{NOTE:} $K = L$ for CSOM \\
        $D$ & unit/neuron size and size of individual input sample \\
        $\mathbf{M} \in \mathcal{R}^{D \times H}$ & synaptic matrix of SOM \\
        $u$ & Best Matching Unit (BMU) \\
        $v_j$ & a non-BMU at location $j$ in CSOM \\
        $V = \{v_j\ |\ 0 < j < K^2\} $ & all non-BMU units in CSOM \\
        $H = \{u\} \cup V$  & all units in the SOM \\
        $\sigma_0$ & initial radius for all neurons \\
        $\lambda_0$ & initial learning rate all neurons \\
        $\sigma_u$ and $\sigma_v$ & radius of BMU and radius of non-BMU neuron\\
        $\lambda_u$ and $\lambda_v$ & learning rate of BMU and learning rate of non-BMU neuron \\
        
        $\omega^2$ & running variance of a neuron\\
        $\mathbf{M}^{\omega^2} \in \mathcal{R}^{D \times H}$ & matrix of running variance of all neurons in SOM \\
        $\sigma_h = \{\sigma_u\} \cup \{\sigma_{v_j}\} \in \mathcal{R}^{1 \times H}$ & radius of all neurons for weight update \\
        $\lambda_h = \{\lambda_u\} \cup \{\lambda_{v_j}\} \in \mathcal{R}^{1 \times H}$ & learning rate for all neurons used in the weight update steps of CSOM \\
        $\lambda^0_{\omega}$ & initial scaling/update factor for updating the running variance of every unit \\
        $\lambda_{\omega}$ & adjusted scaling/update factor for updating the running variance of every unit \\
        $\tau_\sigma$ & time constant for radius \\
        $\tau_\lambda$ & time constant for learning rate \\
        $\mathrm{p}$ & patch size (DendSOM) \\
        $\mathrm{s}$ & stride length (DendSOM) \\
        \hline
    \end{tabular}
    \caption{Notations, Symbols, abbreviations used in this paper}
    \label{tab:symbols}
\end{table}


Note that the DendSOM benchmarks are our implementation of the model, as the original source code was not publicly available.

\subsection{Class Incremental Learning}
This section contains information about class incremental settings where each task size had exactly one class in it.

\subsubsection{MNIST}

\textbf{Hyperparameters:}
Table \ref{tab:mnist_class_hyperparameters} contains the empirically obtained initial hyperparameters for obtaining the best possible results as described in Table \ref{tab:class_incremental}.

\begin{table}[!htb]
    \centering
    \begin{tabular}{ccccccccccc}
        \hline
        Model & $H$ & $D$ & $\sigma_0$ & $\lambda_0$ & $\omega^2_0$ & $\lambda^0_{\omega}$ & $\tau_\sigma$ & $\tau_{\lambda}$ & $\mathrm{p}$ & $\mathrm{s}$ \\
        \hline
        Vanilla SOM & $15 \times 15$ & 28 $\times$ 28 & 0.6 & 0.07 & 1 & 0.9 & 8 & 45 & - & -\\
        DendSOM & 8 $\times$ 8 & 7 $\times$ [14 $\times$ 14] & 4 & 0.95 & 2 & 0.005 & - & - & 10 & 3\\
        CSOM & $15 \times 15$ & $28 \times 28$ & 1.5 & 0.07 & 0.5 & 0.9 & 8 & 45 & - & -\\
        \hline
    \end{tabular}
    \caption{Hyperparameters for Class Incremental setting on MNIST}
    \label{tab:mnist_class_hyperparameters}
\end{table}

\textbf{Task matrices:}
We performed 10 trials of CSOM trained on MNIST. Table \ref{tab:class_incremental_mnist} indicates the resultant accuracy values,

\begin{table}[htb]
    \centering
    \resizebox{\textwidth}{!}{%
    \begin{tabular}{rrrrrrrrrr}
    \hline
     100.0 $\pm$ 0.0 & 0.0 $\pm$ 0.0 & 0.0 $\pm$ 0.0 & 0.0 $\pm$ 0.0 & 0.0 $\pm$ 0.0 & 0.0 $\pm$ 0.0 & 0.0 $\pm$ 0.0 & 0.0 $\pm$ 0.0 & 0.0 $\pm$ 0.0 & 0.0 $\pm$ 0.0 \\
99.77 $\pm$ 0.12 & 99.89 $\pm$ 0.14 & 0.0 $\pm$ 0.0 & 0.0 $\pm$ 0.0 & 0.0 $\pm$ 0.0 & 0.0 $\pm$ 0.0 & 0.0 $\pm$ 0.0 & 0.0 $\pm$ 0.0 & 0.0 $\pm$ 0.0 & 0.0 $\pm$ 0.0 \\
99.17 $\pm$ 0.4 & 99.41 $\pm$ 0.31 & 96.91 $\pm$ 1.2 & 0.0 $\pm$ 0.0 & 0.0 $\pm$ 0.0 & 0.0 $\pm$ 0.0 & 0.0 $\pm$ 0.0 & 0.0 $\pm$ 0.0 & 0.0 $\pm$ 0.0 & 0.0 $\pm$ 0.0 \\
98.83 $\pm$ 0.54 & 99.23 $\pm$ 0.31 & 93.31 $\pm$ 2.98 & 96.01 $\pm$ 1.41 & 0.0 $\pm$ 0.0 & 0.0 $\pm$ 0.0 & 0.0 $\pm$ 0.0 & 0.0 $\pm$ 0.0 & 0.0 $\pm$ 0.0 & 0.0 $\pm$ 0.0 \\
98.78 $\pm$ 0.52 & 99.22 $\pm$ 0.32 & 92.26 $\pm$ 3.38 & 94.61 $\pm$ 1.8 & 96.34 $\pm$ 1.02 & 0.0 $\pm$ 0.0 & 0.0 $\pm$ 0.0 & 0.0 $\pm$ 0.0 & 0.0 $\pm$ 0.0 & 0.0 $\pm$ 0.0 \\
96.78 $\pm$ 3.13 & 99.07 $\pm$ 0.43 & 92.12 $\pm$ 3.39 & 84.99 $\pm$ 9.97 & 96.05 $\pm$ 1.02 & 88.98 $\pm$ 5.21 & 0.0 $\pm$ 0.0 & 0.0 $\pm$ 0.0 & 0.0 $\pm$ 0.0 & 0.0 $\pm$ 0.0 \\
95.96 $\pm$ 3.0 & 98.5 $\pm$ 1.22 & 89.5 $\pm$ 6.05 & 84.52 $\pm$ 10.0 & 95.39 $\pm$ 1.19 & 86.19 $\pm$ 6.37 & 94.8 $\pm$ 2.46 & 0.0 $\pm$ 0.0 & 0.0 $\pm$ 0.0 & 0.0 $\pm$ 0.0 \\
95.93 $\pm$ 2.99 & 98.5 $\pm$ 1.22 & 88.33 $\pm$ 6.72 & 83.39 $\pm$ 10.11 & 93.3 $\pm$ 2.96 & 85.32 $\pm$ 6.34 & 94.75 $\pm$ 2.48 & 92.54 $\pm$ 1.59 & 0.0 $\pm$ 0.0 & 0.0 $\pm$ 0.0 \\
95.93 $\pm$ 3.03 & 98.41 $\pm$ 1.19 & 87.09 $\pm$ 6.55 & 81.73 $\pm$ 9.47 & 93.05 $\pm$ 3.07 & 80.09 $\pm$ 9.43 & 93.6 $\pm$ 4.35 & 92.42 $\pm$ 1.6 & 84.72 $\pm$ 2.25 & 0.0 $\pm$ 0.0 \\
95.89 $\pm$ 3.02 & 98.41 $\pm$ 1.19 & 87.14 $\pm$ 6.55 & 81.39 $\pm$ 9.33 & 79.17 $\pm$ 8.47 & 79.53 $\pm$ 9.35 & 93.36 $\pm$ 4.28 & 80.64 $\pm$ 7.77 & 82.43 $\pm$ 3.17 & 72.37 $\pm$ 14.48 \\
    \hline
    \end{tabular}%
    }
    \caption{mean and standard deviation of accuracies in a task matrix}
    \label{tab:class_incremental_mnist}
\end{table}

\subsubsection{KMNIST}

\textbf{Hyperparameters:}
Table \ref{tab:kmnist_class_hyperparameters} shows the hyperparameters we used for obtaining the results in Table \ref{tab:class_incremental} on KMNIST dataset.

\begin{table}[htb]
    \centering
    \begin{tabular}{ccccccccccc}
        \hline
        Model & $H$ & $D$ & $\sigma_0$ & $\lambda_0$ & $\omega^2_0$ & $\lambda^0_{\omega}$ & $\tau_\sigma$ & $\tau_{\lambda}$ & $\mathrm{p}$ & $\mathrm{s}$ \\
        \hline
        Vanilla SOM & $20 \times 20$ & 28 $\times$ 28 & 0.6 & 0.07 & 1 & 0.9 & 8 & 45 & - & -\\
        DendSOM & 10 $\times$ 10 & 7 $\times$ [14 $\times$ 14] & 6 & 0.95 & 2 & 0.005 & - & - & 4 & 2\\
        CSOM & $35 \times 35$ & $28 \times 28$ & 1.5 & 0.07 & 0.5 & 0.9 & 8 & 45 & - & -\\
        \hline
    \end{tabular}
    \caption{Hyperparameters for Class Incremental setting on KMNIST}
    \label{tab:kmnist_class_hyperparameters}
\end{table}

\textbf{Task matrices:}
We performed 10 trials of CSOM trained on KMNIST. Table \ref{tab:class_incremental_kmnist} indicates the resultant accuracy values,

\begin{table}[htbp]
  \centering
  \resizebox{\textwidth}{!}{%
    \begin{tabular}{rrrrrrrrrr}
    \hline
  100.0 $\pm$ 0.0 & 0.0 $\pm$ 0.0 & 0.0 $\pm$ 0.0 & 0.0 $\pm$ 0.0 & 0.0 $\pm$ 0.0 & 0.0 $\pm$ 0.0 & 0.0 $\pm$ 0.0 & 0.0 $\pm$ 0.0 & 0.0 $\pm$ 0.0 & 0.0 $\pm$ 0.0 \\
98.08 $\pm$ 0.85 & 98.72 $\pm$ 0.71 & 0.0 $\pm$ 0.0 & 0.0 $\pm$ 0.0 & 0.0 $\pm$ 0.0 & 0.0 $\pm$ 0.0 & 0.0 $\pm$ 0.0 & 0.0 $\pm$ 0.0 & 0.0 $\pm$ 0.0 & 0.0 $\pm$ 0.0 \\
97.86 $\pm$ 0.71 & 88.0 $\pm$ 2.71 & 93.25 $\pm$ 1.14 & 0.0 $\pm$ 0.0 & 0.0 $\pm$ 0.0 & 0.0 $\pm$ 0.0 & 0.0 $\pm$ 0.0 & 0.0 $\pm$ 0.0 & 0.0 $\pm$ 0.0 & 0.0 $\pm$ 0.0 \\
97.44 $\pm$ 0.82 & 88.06 $\pm$ 2.81 & 87.59 $\pm$ 1.93 & 94.2 $\pm$ 1.67 & 0.0 $\pm$ 0.0 & 0.0 $\pm$ 0.0 & 0.0 $\pm$ 0.0 & 0.0 $\pm$ 0.0 & 0.0 $\pm$ 0.0 & 0.0 $\pm$ 0.0 \\
93.98 $\pm$ 1.22 & 86.21 $\pm$ 3.18 & 86.63 $\pm$ 1.83 & 93.53 $\pm$ 1.66 & 82.79 $\pm$ 1.95 & 0.0 $\pm$ 0.0 & 0.0 $\pm$ 0.0 & 0.0 $\pm$ 0.0 & 0.0 $\pm$ 0.0 & 0.0 $\pm$ 0.0 \\
89.3 $\pm$ 3.37 & 85.93 $\pm$ 3.2 & 84.0 $\pm$ 3.27 & 91.73 $\pm$ 2.03 & 81.15 $\pm$ 3.03 & 82.6 $\pm$ 3.57 & 0.0 $\pm$ 0.0 & 0.0 $\pm$ 0.0 & 0.0 $\pm$ 0.0 & 0.0 $\pm$ 0.0 \\
88.97 $\pm$ 3.34 & 83.55 $\pm$ 3.24 & 79.92 $\pm$ 3.71 & 91.27 $\pm$ 2.14 & 79.93 $\pm$ 3.2 & 80.48 $\pm$ 3.16 & 85.91 $\pm$ 3.51 & 0.0 $\pm$ 0.0 & 0.0 $\pm$ 0.0 & 0.0 $\pm$ 0.0 \\
86.32 $\pm$ 3.14 & 83.37 $\pm$ 3.24 & 79.42 $\pm$ 3.74 & 91.16 $\pm$ 2.19 & 79.43 $\pm$ 3.17 & 80.33 $\pm$ 3.13 & 85.68 $\pm$ 3.49 & 89.19 $\pm$ 1.55 & 0.0 $\pm$ 0.0 & 0.0 $\pm$ 0.0 \\
83.86 $\pm$ 3.45 & 81.31 $\pm$ 3.21 & 77.98 $\pm$ 3.72 & 90.45 $\pm$ 2.68 & 77.59 $\pm$ 3.58 & 79.34 $\pm$ 3.38 & 83.33 $\pm$ 4.9 & 87.31 $\pm$ 2.63 & 80.75 $\pm$ 4.35 & 0.0 $\pm$ 0.0 \\
83.75 $\pm$ 3.41 & 80.35 $\pm$ 3.18 & 77.21 $\pm$ 3.98 & 90.39 $\pm$ 2.72 & 75.72 $\pm$ 4.39 & 79.28 $\pm$ 3.41 & 83.25 $\pm$ 4.92 & 86.62 $\pm$ 2.2 & 80.79 $\pm$ 4.36 & 78.66 $\pm$ 4.79 \\
    \hline
    \end{tabular}%
    }
    \caption{mean and standard deviation of accuracies in a task matrix}
  \label{tab:class_incremental_kmnist}
\end{table}

\subsubsection{Fashion-MNIST}

\textbf{Hyperparameters:}
Table \ref{tab:fashion_class_hyperparameters} indicates the hyperparameters used for obtaining results in Table \ref{tab:class_incremental} on Fashion-MNIST dataset.

\begin{table}[htb]
    \centering
    \begin{tabular}{ccccccccccc}
        \hline
        Model & $H$ & $D$ & $\sigma_0$ & $\lambda_0$ & $\omega^2_0$ & $\lambda^0_{\omega}$ & $\tau_\sigma$ & $\tau_{\lambda}$ & $\mathrm{p}$ & $\mathrm{s}$ \\
        \hline
        Vanilla SOM & $20 \times 20$ & 28 $\times$ 28 & 0.6 & 0.07 & 1 & 0.9 & 8 & 45 & - & -\\
        DendSOM & 10 $\times$ 10 & 7 $\times$ [14 $\times$ 14] & 5 & 0.95 & 2 & 0.005 & - & - & 8 & 4\\
        CSOM & $25 \times 25$ & $28 \times 28$ & 1.5 & 0.07 & 0.5 & 0.9 & 8 & 45 & - & -\\
        \hline
    \end{tabular}
    \caption{Hyperparameters for Class Incremental setting on Fashion-MNIST}
    \label{tab:fashion_class_hyperparameters}
\end{table}

\begin{table}[!htb]
  \centering
  \resizebox{\textwidth}{!}{%
    \begin{tabular}{rrrrrrrrrr}
    \hline
    100.0 $\pm$ 0.0 & 0.0 $\pm$ 0.0 & 0.0 $\pm$ 0.0 & 0.0 $\pm$ 0.0 & 0.0 $\pm$ 0.0 & 0.0 $\pm$ 0.0 & 0.0 $\pm$ 0.0 & 0.0 $\pm$ 0.0 & 0.0 $\pm$ 0.0 & 0.0 $\pm$ 0.0 \\
97.39 $\pm$ 1.49 & 97.7 $\pm$ 0.59 & 0.0 $\pm$ 0.0 & 0.0 $\pm$ 0.0 & 0.0 $\pm$ 0.0 & 0.0 $\pm$ 0.0 & 0.0 $\pm$ 0.0 & 0.0 $\pm$ 0.0 & 0.0 $\pm$ 0.0 & 0.0 $\pm$ 0.0 \\
92.9 $\pm$ 1.91 & 96.72 $\pm$ 0.82 & 94.32 $\pm$ 0.71 & 0.0 $\pm$ 0.0 & 0.0 $\pm$ 0.0 & 0.0 $\pm$ 0.0 & 0.0 $\pm$ 0.0 & 0.0 $\pm$ 0.0 & 0.0 $\pm$ 0.0 & 0.0 $\pm$ 0.0 \\
85.83 $\pm$ 3.25 & 88.44 $\pm$ 6.73 & 93.67 $\pm$ 0.72 & 89.49 $\pm$ 3.6 & 0.0 $\pm$ 0.0 & 0.0 $\pm$ 0.0 & 0.0 $\pm$ 0.0 & 0.0 $\pm$ 0.0 & 0.0 $\pm$ 0.0 & 0.0 $\pm$ 0.0 \\
85.71 $\pm$ 2.99 & 88.34 $\pm$ 6.68 & 73.42 $\pm$ 4.09 & 83.11 $\pm$ 4.61 & 76.2 $\pm$ 3.7 & 0.0 $\pm$ 0.0 & 0.0 $\pm$ 0.0 & 0.0 $\pm$ 0.0 & 0.0 $\pm$ 0.0 & 0.0 $\pm$ 0.0 \\
85.64 $\pm$ 3.03 & 88.34 $\pm$ 6.68 & 73.49 $\pm$ 4.08 & 83.07 $\pm$ 4.58 & 75.93 $\pm$ 3.79 & 99.76 $\pm$ 0.11 & 0.0 $\pm$ 0.0 & 0.0 $\pm$ 0.0 & 0.0 $\pm$ 0.0 & 0.0 $\pm$ 0.0 \\
77.5 $\pm$ 3.31 & 88.34 $\pm$ 6.69 & 68.65 $\pm$ 3.46 & 82.88 $\pm$ 4.54 & 72.45 $\pm$ 3.56 & 99.6 $\pm$ 0.39 & 41.62 $\pm$ 2.74 & 0.0 $\pm$ 0.0 & 0.0 $\pm$ 0.0 & 0.0 $\pm$ 0.0 \\
77.51 $\pm$ 3.32 & 88.34 $\pm$ 6.69 & 68.65 $\pm$ 3.48 & 82.88 $\pm$ 4.54 & 72.45 $\pm$ 3.56 & 66.87 $\pm$ 13.25 & 41.55 $\pm$ 2.73 & 95.77 $\pm$ 3.74 & 0.0 $\pm$ 0.0 & 0.0 $\pm$ 0.0 \\
77.24 $\pm$ 3.24 & 88.33 $\pm$ 6.68 & 68.57 $\pm$ 3.48 & 82.62 $\pm$ 4.61 & 72.11 $\pm$ 3.42 & 66.83 $\pm$ 13.1 & 39.95 $\pm$ 2.75 & 90.12 $\pm$ 9.74 & 94.84 $\pm$ 0.94 & 0.0 $\pm$ 0.0 \\
77.23 $\pm$ 3.21 & 88.34 $\pm$ 6.66 & 68.53 $\pm$ 3.47 & 82.6 $\pm$ 4.59 & 72.08 $\pm$ 3.4 & 58.84 $\pm$ 14.3 & 39.87 $\pm$ 2.69 & 76.89 $\pm$ 11.12 & 94.28 $\pm$ 1.19 & 92.65 $\pm$ 2.89 \\
    \hline
    \end{tabular}%
    }
    \caption{mean and standard deviation of accuracies in a task matrix}
  \label{tab:class_incremental_fashion}
\end{table}

\subsubsection{CIFAR10}

\textbf{Hyperparameters:}
Table \ref{tab:cifar_class_hyperparameters} indicates the hyperparameters used to obtain results in Table \ref{tab:cifar} the on split-CIFAR10 dataset.

\begin{table}[!htb]
    \centering
    \begin{tabular}{ccccccccc}
        \hline
        Model & $H$ & $D$ & $\sigma_0$ & $\lambda_0$ & $\omega^2_0$ & $\lambda^0_{\omega}$ & $\tau_\sigma$ & $\tau_{\lambda}$ \\
        \hline
        CSOM & $100 \times 100$ & $32 \times 32$ & 1.5 & 0.2 & 0.6 & 0.9 & 6 & 45 \\
        \hline
    \end{tabular}
    \caption{Hyperparameters for CSOM for Class Incremental setting on CIFAR10}
    \label{tab:cifar_class_hyperparameters}
\end{table}

Table \ref{tab:pycil_hyperparameters} indicates hyperparameters set while implementing models from \cite{zhou2023pycil} in a Class Incremental Online Learning setting.

\begin{table}[!htb]
    \centering
    \begin{tabular}{cccccc}
        \hline
        Papers & Memory Size & Memory per Class & Model & epochs & Fixed Memory \\
        \hline
        Finetune, Replay, LwF,  & 2000 & 20 & Resnet32 & 1 & \textcolor{red}{\xmark} \\
        iCarl, EWC, Memo, Podnet, BiC &  &  &  &  &  \\
        \hline
    \end{tabular}
    \caption{Hyperparameters for benchmarks from PyCIL \cite{zhou2023pycil} on CIFAR10}
    \label{tab:pycil_hyperparameters}
\end{table}

\begin{table}[!htb]
    \centering
    \begin{tabular}{cccccc}
        \hline
        Paper & Model & Data setting & lr & epochs & memory size \\
        \hline
        SCALE & Resnet18 & \textit{seq} & 0.03 & 1 & 256 \\
        \hline
    \end{tabular}
    \caption{Hyperparameters for benchmarks from SCALE \cite{SCALE} using SGD optimizer}
    \label{tab:scale_hyperparameters}
\end{table}

Apart from the hyperparameters mentioned in Tables \ref{tab:pycil_hyperparameters} and \ref{tab:scale_hyperparameters}, the rest of the hyperparameters were unchanged in their publicly available code. We added an evaluation code for obtaining task matrices for all the models in Table \ref{tab:cifar}.

\textbf{Task matrices:}
We performed 8 trials of CSOM trained on CIFAR-10. Following resultant accuracies values.

\begin{table}[!htbp]
  \centering
  \resizebox{\textwidth}{!}{%
    \begin{tabular}{rrrrrrrrrr}
    \hline
    100.0 $\pm$ 0.0 & 0.0 $\pm$ 0.0 & 0.0 $\pm$ 0.0 & 0.0 $\pm$ 0.0 & 0.0 $\pm$ 0.0 & 0.0 $\pm$ 0.0 & 0.0 $\pm$ 0.0 & 0.0 $\pm$ 0.0 & 0.0 $\pm$ 0.0 & 0.0 $\pm$ 0.0 \\
84.2 $\pm$ 1.55 & 76.88 $\pm$ 1.22 & 0.0 $\pm$ 0.0 & 0.0 $\pm$ 0.0 & 0.0 $\pm$ 0.0 & 0.0 $\pm$ 0.0 & 0.0 $\pm$ 0.0 & 0.0 $\pm$ 0.0 & 0.0 $\pm$ 0.0 & 0.0 $\pm$ 0.0 \\
65.86 $\pm$ 1.37 & 73.09 $\pm$ 1.15 & 52.61 $\pm$ 2.5 & 0.0 $\pm$ 0.0 & 0.0 $\pm$ 0.0 & 0.0 $\pm$ 0.0 & 0.0 $\pm$ 0.0 & 0.0 $\pm$ 0.0 & 0.0 $\pm$ 0.0 & 0.0 $\pm$ 0.0 \\
62.0 $\pm$ 1.42 & 67.05 $\pm$ 1.07 & 41.47 $\pm$ 2.12 & 46.31 $\pm$ 1.57 & 0.0 $\pm$ 0.0 & 0.0 $\pm$ 0.0 & 0.0 $\pm$ 0.0 & 0.0 $\pm$ 0.0 & 0.0 $\pm$ 0.0 & 0.0 $\pm$ 0.0 \\
55.48 $\pm$ 1.41 & 65.26 $\pm$ 1.14 & 36.68 $\pm$ 1.92 & 43.8 $\pm$ 1.55 & 22.42 $\pm$ 1.16 & 0.0 $\pm$ 0.0 & 0.0 $\pm$ 0.0 & 0.0 $\pm$ 0.0 & 0.0 $\pm$ 0.0 & 0.0 $\pm$ 0.0 \\
53.9 $\pm$ 1.59 & 63.2 $\pm$ 0.91 & 32.28 $\pm$ 2.28 & 32.22 $\pm$ 1.68 & 22.2 $\pm$ 0.42 & 31.49 $\pm$ 1.44 & 0.0 $\pm$ 0.0 & 0.0 $\pm$ 0.0 & 0.0 $\pm$ 0.0 & 0.0 $\pm$ 0.0 \\
52.76 $\pm$ 1.34 & 61.82 $\pm$ 0.7 & 30.75 $\pm$ 2.13 & 29.7 $\pm$ 1.89 & 21.41 $\pm$ 0.83 & 31.01 $\pm$ 1.5 & 20.24 $\pm$ 1.21 & 0.0 $\pm$ 0.0 & 0.0 $\pm$ 0.0 & 0.0 $\pm$ 0.0 \\
50.96 $\pm$ 0.77 & 59.95 $\pm$ 1.06 & 29.02 $\pm$ 1.74 & 28.08 $\pm$ 1.74 & 19.91 $\pm$ 0.96 & 29.31 $\pm$ 1.66 & 19.79 $\pm$ 1.17 & 37.3 $\pm$ 1.09 & 0.0 $\pm$ 0.0 & 0.0 $\pm$ 0.0 \\
42.49 $\pm$ 1.07 & 53.13 $\pm$ 1.29 & 27.98 $\pm$ 1.74 & 27.06 $\pm$ 1.25 & 19.11 $\pm$ 1.05 & 28.64 $\pm$ 1.48 & 19.32 $\pm$ 1.05 & 37.14 $\pm$ 0.99 & 38.04 $\pm$ 1.61 & 0.0 $\pm$ 0.0 \\
41.2 $\pm$ 1.11 & 42.96 $\pm$ 1.1 & 27.71 $\pm$ 1.59 & 26.19 $\pm$ 1.15 & 18.74 $\pm$ 1.12 & 28.51 $\pm$ 1.18 & 18.88 $\pm$ 0.96 & 36.25 $\pm$ 1.04 & 36.36 $\pm$ 1.65 & 36.01 $\pm$ 0.9 \\
    \hline
    \end{tabular}%
    }
    \caption{mean and standard deviation of accuracies in a task matrix}
  \label{tab:class_incremental_cifar}
\end{table}

\subsection{Domain Incremental Learning}

For the domain incremental setting, we set number of tasks = 5, number of classes/task = 2

\subsubsection{MNIST}

\textbf{Hyperparameters:}
Table \ref{tab:mnist_domain_hyperparameters} contains the initial hyperparameters used for obtaining the results shown in Table \ref{tab:domain_incremental}.

\begin{table}[!htb]
    \centering
    \begin{tabular}{ccccccccccc}
        \hline
        Model & $H$ & $D$ & $\sigma_0$ & $\lambda_0$ & $\omega^2_h$ & $\lambda^0_{\omega}$ & $\tau_\sigma$ & $\tau_{\lambda}$ & $\mathrm{p}$ & $\mathrm{s}$ \\
        \hline
        Vanilla SOM & $15 \times 15$ & 28 $\times$ 28 & 0.6 & 0.07 & 1 & 0.9 & 8 & 45 & - & -\\
        DendSOM & 8 $\times$ 8 & 7 $\times$ [14 $\times$ 14] & 4 & 0.95 & 2 & 0.005 & - & - & 10 & 3\\
        CSOM & $15 \times 15$ & $28 \times 28$ & 1.5 & 0.07 & 0.5 & 0.9 & 8 & 45 & - & -\\
        \hline
    \end{tabular}
    \caption{Hyperparameters for Domain Incremental setting on MNIST}
    \label{tab:mnist_domain_hyperparameters}
\end{table}

\textbf{Task matrices:}
Table \ref{tab:domain_mnist} indicates the resultant accuracy values obtained after performing 10 trials.

\begin{table}[!htbp]
    \centering
    \begin{tabular}{rrrrr}
    \hline
    99.85 $\pm$ 0.04 & 45.59 $\pm$ 5.94 & 45.63 $\pm$ 3.38 & 68.99 $\pm$ 4.3 & 39.82 $\pm$ 5.51 \\
99.04 $\pm$ 1.04 & 93.57 $\pm$ 2.13 & 61.77 $\pm$ 3.29 & 73.77 $\pm$ 10.54 & 45.9 $\pm$ 4.1 \\
97.55 $\pm$ 1.65 & 93.25 $\pm$ 2.75 & 93.88 $\pm$ 2.96 & 55.48 $\pm$ 7.08 & 18.41 $\pm$ 3.85 \\
98.1 $\pm$ 1.65 & 92.6 $\pm$ 2.78 & 90.52 $\pm$ 5.21 & 96.94 $\pm$ 0.97 & 29.71 $\pm$ 4.53 \\
97.99 $\pm$ 1.65 & 91.85 $\pm$ 2.77 & 77.5 $\pm$ 9.88 & 97.2 $\pm$ 0.87 & 80.92 $\pm$ 6.1 \\
    \hline
    \end{tabular}
    \caption{mean and standard deviation of accuracies in a task matrix}
  \label{tab:domain_mnist}
\end{table}

\subsubsection{KMNIST}

\textbf{Hyperparameters:}
Table \ref{tab:kmnist_domain_hyperparameters} contains the initial hyperparameters used for obtaining the results shown in Table \ref{tab:domain_incremental}.

\begin{table}[!htb]
    \centering
    \begin{tabular}{ccccccccccc}
        \hline
        Model & $H$ & $D$ & $\sigma_0$ & $\lambda_0$ & $\omega^2_h$ & $\lambda^0_{\omega}$ & $\tau_\sigma$ & $\tau_{\lambda}$ & $\mathrm{p}$ & $\mathrm{s}$ \\
        \hline
        Vanilla SOM & $20 \times 20$ & 28 $\times$ 28 & 0.6 & 0.07 & 1 & 0.9 & 8 & 45 & - & -\\
        DendSOM & 12 $\times$ 12 & 7 $\times$ [14 $\times$ 14] & 6 & 0.95 & 2 & 0.005 & - & - & 4 & 2 \\
        CSOM & $35 \times 35$ & $28 \times 28$ & 1.5 & 0.07 & 0.5 & 0.9 & 8 & 45 & - & -\\
        \hline
    \end{tabular}
    \caption{Hyperparameters for Domain Incremental setting on KMNIST}
    \label{tab:kmnist_domain_hyperparameters}
\end{table}

\textbf{Task matrices:}
Table \ref{tab:domain_kmnist} indicates the resultant accuracy values obtained after performing 10 trials

\begin{table}[!tbp]
  \centering
  \begin{tabular}{rrrrr}
    \hline
    98.6 $\pm$ 0.4 & 29.35 $\pm$ 1.24 & 54.72 $\pm$ 2.57 & 28.63 $\pm$ 2.45 & 47.21 $\pm$ 2.47 \\
92.19 $\pm$ 2.46 & 91.48 $\pm$ 1.68 & 41.84 $\pm$ 3.3 & 55.72 $\pm$ 3.25 & 42.46 $\pm$ 2.05 \\
89.91 $\pm$ 2.75 & 89.74 $\pm$ 1.99 & 89.85 $\pm$ 2.09 & 52.24 $\pm$ 2.62 & 39.38 $\pm$ 1.62 \\
86.8 $\pm$ 2.5 & 89.55 $\pm$ 1.94 & 88.32 $\pm$ 2.21 & 93.01 $\pm$ 1.83 & 45.63 $\pm$ 1.43 \\
86.1 $\pm$ 2.69 & 88.89 $\pm$ 2.32 & 86.72 $\pm$ 3.45 & 91.91 $\pm$ 2.87 & 86.92 $\pm$ 2.3 \\
    \hline
  \end{tabular}
  \caption{mean and standard deviation of accuracies in a task matrix}
  \label{tab:domain_kmnist}
\end{table}

\subsubsection{Fashion-MNIST}

\textbf{Hyperparameters:}
Table \ref{tab:fashion_domain_hyperparameters} contains the initial hyperparameters used for obtaining the results shown in Table \ref{tab:domain_incremental}.

\begin{table}[!tb]
    \centering
    \begin{tabular}{ccccccccccc}
        \hline
        Model & $H$ & $D$ & $\sigma_0$ & $\lambda_0$ & $\omega^2_h$ & $\lambda^0_{\omega}$ & $\tau_\sigma$ & $\tau_{\lambda}$ & $\mathrm{p}$ & $\mathrm{s}$ \\
        \hline
        Vanilla SOM & $20 \times 20$ & 28 $\times$ 28 & 0.6 & 0.07 & 1 & 0.9 & 8 & 45 & - & -\\
        DendSOM & 10 $\times$ 10 & 7 $\times$ [14 $\times$ 14] & 5 & 0.95 & 2 & 0.005 & - & - & 8 & 4 \\
        CSOM & $25 \times 25$ & $28 \times 28$ & 1.5 & 0.07 & 0.5 & 0.9 & 8 & 45 & - & -\\
        \hline
    \end{tabular}
    \caption{Hyperparameters for Domain Incremental setting on Fashion-MNIST}
    \label{tab:fashion_domain_hyperparameters}
\end{table}

\textbf{Task matrices:}
Table \ref{tab:domain_fashion} indicates the resultant accuracy values obtained after performing 10 trials.

\begin{table}[!ht]
  \centering
    \begin{tabular}{rrrrr}
    \hline
     96.72 $\pm$ 1.44 & 53.15 $\pm$ 6.27 & 28.88 $\pm$ 4.84 & 40.7 $\pm$ 2.94 & 49.77 $\pm$ 0.27 \\
92.63 $\pm$ 1.79 & 95.08 $\pm$ 0.61 & 41.42 $\pm$ 1.21 & 44.46 $\pm$ 1.51 & 47.21 $\pm$ 2.26 \\
92.94 $\pm$ 1.72 & 92.22 $\pm$ 1.61 & 97.55 $\pm$ 0.87 & 95.75 $\pm$ 1.36 & 83.07 $\pm$ 3.94 \\
93.72 $\pm$ 1.71 & 91.37 $\pm$ 1.49 & 97.58 $\pm$ 0.92 & 97.22 $\pm$ 0.86 & 86.17 $\pm$ 4.71 \\
93.96 $\pm$ 1.73 & 90.77 $\pm$ 2.15 & 97.1 $\pm$ 0.85 & 97.06 $\pm$ 1.22 & 99.13 $\pm$ 0.13 \\
    \hline
    \end{tabular}
    \caption{mean and standard deviation of accuracies in a task matrix}
  \label{tab:domain_fashion}
\end{table}

\subsubsection{CIFAR-10}

\textbf{Hyperparameters:}
Table \ref{tab:cifar_domain_hyperparameters} shows hyperparameters used for training CSOM on CIFAR10 in a Domain Incremental Setting.

\begin{table}[!htb]
    \centering
    \begin{tabular}{ccccccccccc}
        \hline
        Model & $H$ & $D$ & $\sigma_0$ & $\lambda_0$ & $\omega^2_h$ & $\lambda^0_{\omega}$ & $\tau_\sigma$ & $\tau_{\lambda}$ \\
        \hline
        CSOM & $15 \times 15$ & $28 \times 28$ & 1.5 & 0.07 & 0.5 & 0.9 & 8 & 45 \\
        \hline
    \end{tabular}
    \caption{Hyperparameters for Domain Incremental setting on CIFAR10}
    \label{tab:cifar_domain_hyperparameters}
\end{table}

Table \ref{tab:domain_cifar} shows task matrix of accuracy values obtained after performing 8 trials of CSOM on grayscale images from CIFAR-10.

\begin{table}[!htbp]
  \centering
    \begin{tabular}{rrrrr}
    \hline
     99.52 $\pm$ 0.37 & 46.01 $\pm$ 4.66 & 43.63 $\pm$ 5.55 & 65.38 $\pm$ 5.11 & 38.67 $\pm$ 6.55 \\
98.79 $\pm$ 0.59 & 93.15 $\pm$ 1.09 & 60.14 $\pm$ 2.78 & 74.98 $\pm$ 7.41 & 50.55 $\pm$ 2.84 \\
97.15 $\pm$ 1.38 & 91.67 $\pm$ 1.84 & 94.41 $\pm$ 1.08 & 52.1 $\pm$ 7.59 & 18.99 $\pm$ 3.8 \\
95.53 $\pm$ 7.62 & 91.35 $\pm$ 1.74 & 89.13 $\pm$ 5.01 & 95.36 $\pm$ 0.72 & 30.34 $\pm$ 3.01 \\
97.25 $\pm$ 2.91 & 91.23 $\pm$ 0.92 & 68.91 $\pm$ 5.9 & 95.85 $\pm$ 0.77 & 79.57 $\pm$ 5.46 \\
    \hline
    \end{tabular}
    \caption{mean and standard deviation of accuracies in a task matrix}
  \label{tab:domain_cifar}
\end{table}

\end{document}